\documentclass[10pt]{article}

\usepackage[numbers, compress, sort]{natbib}
\usepackage[utf8]{inputenc} 
\usepackage[T1]{fontenc}    
\usepackage{url}            
\usepackage{booktabs}       
\usepackage{amsfonts}       
\usepackage{nicefrac}       
\usepackage{microtype}
\usepackage{xcolor}
\usepackage{array}
\usepackage{paralist}
\usepackage[left=1in,right=1in,top=1in,bottom=1in]{geometry}

\usepackage{bbm}
\usepackage{multirow}
\usepackage{amsthm,amssymb,amsmath}
\usepackage{mathtools}
\usepackage{graphicx}
\usepackage{thmtools,thm-restate}
\newcommand{\R}{\mathbb{R}}
\newcommand{\E}{\mathbb{E}}

\DeclarePairedDelimiter{\brk}{[}{]}
\DeclarePairedDelimiter{\crl}{\{}{\}}
\DeclarePairedDelimiter{\prn}{(}{)}
\DeclarePairedDelimiter{\nrm}{\|}{\|}

\DeclarePairedDelimiter{\ceil}{\lceil}{\rceil}

\DeclareMathOperator*{\argmin}{arg\,min}

\newcommand{\mc}[1]{\mathcal{#1}}

\renewcommand\sb[1]{\left[#1\right]}
\newcommand\cb[1]{\left\{#1\right\}}

\def\ddefloop#1{\ifx\ddefloop#1\else\ddef{#1}\expandafter\ddefloop\fi}
\def\ddef#1{\expandafter\def\csname 
bb#1\endcsname{\ensuremath{\mathbb{#1}}}}
\ddefloop ABCDEFGHIJKLMNOPQRSTUVWXYZ\ddefloop
\def\ddefloop#1{\ifx\ddefloop#1\else\ddef{#1}\expandafter\ddefloop\fi}
\def\ddef#1{\expandafter\def\csname 
b#1\endcsname{\ensuremath{\mathbf{#1}}}}
\ddefloop ABCDEFGHIJKLMNOPQRSTUVWXYZ\ddefloop
\def\ddef#1{\expandafter\def\csname 
c#1\endcsname{\ensuremath{\mathcal{#1}}}}
\ddefloop ABCDEFGHIJKLMNOPQRSTUVWXYZ\ddefloop
\def\ddef#1{\expandafter\def\csname 
h#1\endcsname{\ensuremath{\widehat{#1}}}}
\ddefloop ABCDEFGHIJKLMNOPQRSTUVWXYZ\ddefloop
\def\ddef#1{\expandafter\def\csname 
hc#1\endcsname{\ensuremath{\widehat{\mathcal{#1}}}}}
\ddefloop ABCDEFGHIJKLMNOPQRSTUVWXYZ\ddefloop
\def\ddef#1{\expandafter\def\csname 
t#1\endcsname{\ensuremath{\widetilde{#1}}}}
\ddefloop ABCDEFGHIJKLMNOPQRSTUVWXYZ\ddefloop
\def\ddef#1{\expandafter\def\csname 
tc#1\endcsname{\ensuremath{\widetilde{\mathcal{#1}}}}}
\ddefloop ABCDEFGHIJKLMNOPQRSTUVWXYZ\ddefloop

\newcommand{\indicatorb}[1]{\mathbbm{1}{\crl*{#1}}}    

\usepackage{nicefrac}

\newsavebox\CBox

\newcommand{\spn}{\mathrm{span}}

\newcommand{\inner}[2]{\left\langle #1,\, #2 \right\rangle}

\usepackage[colorlinks,allcolors=blue]{hyperref}       

\newtheorem{lemma}{Lemma}

\newtheorem{theorem}{Theorem}
\newtheorem{definition}{Definition}
\newtheorem{assumption}{Assumption}

\DeclareMathOperator*{\supp}{supp}

\newcommand{\bx}[0]{\bar{x}}

\newcommand{\pp}[1]{\left[#1 \right]_+}

\newcommand{\sdiff}[0]{\zeta_*}

\newcommand{\collapse}[1]{$\dots$}
\renewcommand{\indicatorb}[1]{\mathbbm{1}_{\crl*{#1}}}

\newcommand{\removed}[1]{}

\newcommand{\ignore}[1]{}

\title{\vspace{-2em}\rule{\linewidth}{2pt}\\ \textbf{Minibatch vs Local SGD for \\Heterogeneous Distributed Learning} \\\rule[8pt]{\linewidth}{1pt}\vspace{-0.7em}}
\author{\normalsize
\begin{minipage}{0.23\textwidth}
\centering
\textbf{Blake Woodworth}\\ 
Toyota Technological\\
Institute at Chicago \\
\small\url{blake@ttic.edu}
\end{minipage}
\begin{minipage}{0.28\textwidth}
\centering
\textbf{Kumar Kshitij Patel} \\
Toyota Technological \\
Institute at Chicago \\
\small\url{kkpatel@ttic.edu} 
\end{minipage}
\begin{minipage}{0.23\textwidth}
\centering
\textbf{Nathan Srebro} \\
Toyota Technological \\
Institute at Chicago \\
\small\url{nati@ttic.edu} 
\end{minipage}
}
\date{}
\begin{document}
\maketitle
\begin{abstract}
We analyze Local SGD (aka parallel or federated SGD) and Minibatch SGD in the heterogeneous distributed setting, where each machine has access to stochastic gradient estimates for a different, machine-specific, convex objective; the goal is to optimize w.r.t.~the average objective; and machines can only communicate intermittently.  We argue that, (i) Minibatch SGD (even without acceleration) dominates all existing analysis of Local SGD in this setting, (ii) accelerated Minibatch SGD is optimal when the heterogeneity is high, and (iii) present the first upper bound for Local SGD that improves over Minibatch SGD in a non-homogeneous regime. 
\end{abstract}

\section{Introduction}\label{sec:intro}

Given the massive scale of many modern machine learning models and datasets, it has become important to develop better methods for distributed training. A particularly important setting for distributed stochastic optimization/learning, and the one we will consider in this work is characterized by, (1) training data that is distributed across many parallel devices rather than centralized in a single node; (2) this data is distributed \emph{heterogeneously}, meaning that each individual machine has data drawn from a \emph{different} distribution; and (3) the frequency of communication between the devices is limited. The goal is to find a single consensus predictor that performs well on all the local distributions simultaneously \cite{boyd2011distributed,bertsekas1989parallel}. The heterogeneity of the data significantly increases the difficulty of distributed learning because the machines' local objectives may be completely different, so a perfect model for one distribution might be terrible for all the others. Limited communication between devices can make it even more challenging to find a good consensus.

One possible approach is using Minibatch Stochastic Gradient Descent (SGD). Between communications, each machine computes one large minibatch stochastic gradient using its local data; then the machines average their local minibatch gradients, yielding one extra-large minibatch gradient comprising data from all the machines' local distributions, which is used for a single SGD update. Minibatch SGD can also be accelerated to improve its convergence rate \cite{ghadimi2012optimal,ghadimi2013optimal}. 
This algorithm is simple, ubiquitous, and performs very successfully in a variety of settings. 

However, there has recently been great interest in another algorithm, Local SGD (also known as Parallel SGD or Federated SGD) \cite{mcmahan2016communication,zinkevich2010parallelized,stich2018local}, which has been suggested as an improvement over the na\"ive approach of Minibatch SGD. For Local SGD, each machine independently runs SGD on its local objective and, each time they communicate, the machines average together their local iterates.
Local SGD is a very appealing approach---unlike Minibatch SGD, each machine is constantly improving its local model's performance on the local objective, even when the machines are not communicating with each other, so the number of updates is decoupled from the number of communications. In addition to the intuitive benefits, Local SGD has also performed well in many applications \cite{zhang2016parallel,lin2018don,Zhou2018:Kaveraging}.

But can we show that Local SGD is in fact better then Minibatch SGD for convex, heterogeneous objectives?  Does it enjoy better guarantees, and in what settings?  To answer these questions we need to analyze the performance of Local SGD and Minibach SGD.

A number of recent papers have analyzed the convergence properties of Local SGD in the heterogeneous data setting \cite[e.g.][]{wang2018cooperative,karimireddy2019scaffold,khaled2019better,koloskova2020unified}. But, as we will discuss, none of these Local SGD guarantees can show improvement over Minibatch SGD, even without acceleration, and in many regimes they are much worse.  Is this a weakness of the analysis? Can the bounds be improved to show that Local SGD is actually better than Minibatch SGD in certain regimes? Or even at least as good?

Until recently, the situation was similar also for homogeneous distributed optimization, where all the machines' data distributions are the same. There were many published analyses of Local SGD, none of which showed any improvement over Minibatch SGD, or even matched Minibatch SGD's performance.  Only recently \citet{woodworth2020local} settled the issue for the homogeneous case, showing that on the one hand, when communication is rare, Local SGD provably outperforms even accelerated Minibatch SGD, but on the other, Local SGD does not always match Minibatch SGD's performance guarantees and in some situations is provably {\em worse} than Minibatch SGD.

How does this situation play out in the more challenging, and perhaps more interesting, heterogeneous setting?  Some have suggested that the more difficult heterogeneous setting is where we should expect Local SGD to really shine, and where its analysis becomes even more relevant. 
So, how does heterogeneity affect both Local SGD and Minibatch SGD, and the comparison between them?  Do any existing analyses of Local SGD show improvement over Minibatch SGD in the heterogeneous setting?  Is Local SGD still better than Minibatch SGD when communication is rare?  Does the added complexity of heterogeneity perhaps necessitate the more sophisticated Local SGD approach, as some have suggested?  What is the optimal method in this more challenging setting?

In fact, \citet{karimireddy2019scaffold} recently argued that heterogeneity can be particularly problematic for Local-SGD, proving a lower bound for it that shows degradation as heterogeneity increases.  In Section \ref{sec:local-sgd}, we discuss how this lower bound implies that Local SGD is strictly {\em worse} than Minibatch SGD when the level of heterogeneity is very large. However, even with \citeauthor{karimireddy2019scaffold}'s lower bound, it is not clear whether or not Local SGD can improve over Minibatch SGD for merely moderately heterogeneous objectives.

In this paper, we expand on \citeauthor{karimireddy2019scaffold}'s observation about the ill-suitability of Local SGD to the heterogeneous setting.  We prove that existing analysis of Local SGD for heterogeneous data cannot be substantially improved unless the setting is very near-homogeneous.  This is disappointing for Local SGD, because it indicates that unless the level of heterogeneity is very low, the performance it can ensure is truly worse than Minibatch SGD, even without acceleration and regardless of the frequency of communication. At the same time, we provide a more refined analysis of Local SGD showing that Local SGD does, in fact, improve over Minibatch SGD when the level of heterogeneity is sufficiently small (i.e.~the problem is not exactly homogeneous, but it is at least near-homogeneous). This is the first result to show that Local SGD improves over Minibatch SGD in \emph{any} non-homogeneous setting. 

\ignore{Unlike the situation described above for near-homogeneity,} We further show that with even moderately high heterogeneity (or when we do not restrict the dissimilarity between machines), Accelerated Minibatch SGD is in fact optimal for heterogeneous stochastic distributed optimization!  This is because, as we show, Minibatch SGD and its accelerated variant are \ignore{in some sense,}immune to the heterogeneity of the problem.  Perhaps the most important conclusion of our study is that we identify a regime, in which the data is moderately heterogeneous, where Accelerated Minibatch SGD may \emph{not} be optimal and where Local SGD is certainly worse than Minibatch SGD, and so new methods may be needed.

\section{Setup}\label{sec:setup}


We consider heterogeneous distributed stochastic convex optimization/learning using $M$ machines, each of which has access to its own data distribution $\mc{D}^m$. 
The goal is to find an approximate minimizer of the average of the local objectives:
\begin{equation}\label{eq:objective}
\min_{x\in\R^d} F(x) := \frac{1}{M}\sum_{m=1}^M F_m(x) := \frac{1}{M}\sum_{m=1}^M \underset{z^m \sim \mc{D}^m}{\E} f(x;z^m)
\end{equation}
This objective captures, for example, supervised learning where $z = (z_\textrm{features}, z_\textrm{label})$ and $f(x;z)$ is the loss of the predictor, parametrized by $x$, on the instance $z$.  The per-machine distribution $\mc{D}^m$ can be thought of as the empirical distribution of data on machine $m$, or as source distribution which varies between servers, regions or devices.

We focus on a setting in which each machine performs local computations using samples from its own distribution, and is able to communicate with other machines periodically in order to build consensus. This situation arises, for example, when communication is expensive relative to local computation, so it is advantageous to limit the frequency of communication to improve performance.

Concretely, we consider distributed first-order algorithms where each machine computes a total of $T$ stochastic gradients, and is limited to communicate with the others $R$ times. We divide the optimization process into $R$ rounds, where each round consists of each machine calculating and processing $K = T/R$ stochastic gradients, and then communicating with all other machines\footnote{Variable-length rounds of communication are also possible, but do not substantially change the picture.}. Each stochastic gradient for machine $m$ is given by $\nabla f(x;z^m)$ for an independent $z^m \sim \mc{D}^m$.

A simple algorithm for this setting is {\bf Minibatch SGD}. During each round, each machine computes $K$ stochastic gradients $\{g^m_{r,k}\}_{k\in [K]}$ at the same point $x_{r}$, and communicates its average $g_r^m$. Then, the averages from all machines are averaged, to obtain an estimate $g_r$ of the gradient of the overall objective. The estimate $g_r$ is based on all $KM$ stochastic gradients, and is used to obtain the iterate $x_{r+1}$.  Overall, we perform $R$ steps of SGD, each step based on a mini-batch of $KM$ (non-i.i.d.) samples.  To summarize, initializing at some $x_0$, Minibatch SGD operates as follows,
\begin{equation}\label{eq:mbsgd-updates}
\begin{aligned}
g^m_{r,k} &= \nabla f(x_{r}; z_{r,k}^m),\;\; z_{r,k}^m \sim \mc{D}^m,\;\; m=1\ldots M, k=0\ldots K-1,\\
g_r &= \frac{1}{M}\sum_{m=1}^{M} g^m_r \quad \textrm{where} \; g^m_r = \frac{1}{K}\sum_{k=1}^{K} g^m_{r,k},\\
x_{r+1} &= x_{r} - \eta_{r}g_r,
\end{aligned}
\end{equation}
Alternatively, we can perform the same stochastic gradient calculation and aggregation of $g_r$ but replace the simple SGD updates with more sophisticated updates and carefully tuned momentum parameters. 
Throughout this paper, we use ``{\bf Accelerated Minibatch SGD}'' to refer to two accelerated variants of SGD: AC-SA \cite{ghadimi2012optimal} for convex objectives, and multi-stage AC-SA \cite{ghadimi2013optimal} for strongly convex objectives. Algorithmic details including pseudo-code are provided in Appendix \ref{sec:ac_mbsgd_rates}. 

Unlike Minibatch SGD, where each machine spends an entire round calculating stochastic gradients at the same point, \textbf{Local SGD} allows the machines to update local iterates throughout the round based on their own stochastic gradient estimates. Each round starts with a common iterate on all machines, then each machine executes $K$ steps of SGD on its own local objective and communicates the final iterate. These iterates are averaged to form the starting point for the next round. Using $x^m_{r,k}$ to denote the iterate after $r$ rounds and $k$ local steps on machine $m$ and initializing at $x^m_{0,0} = x_0$ for all $m\in [M]$, Local SGD operates as follows,
\begin{equation}\label{eq:localsgd-updates}
\begin{aligned}
g^m_{r,k} &:= \nabla f(x^m_{r,k}; z_{r,k}^m),\;\; z_{r,k}^m \sim \mc{D}^m,\;\; m=1\ldots M, k=0\ldots K-1,\\
x_{r,k+1}^m &= x_{r,k}^m - \eta_{r,k} g^m_{r,k}\\
x_{r+1,0}^m &= x_{r+1} := \frac{1}{M}\sum_{m=1}^{M} x^m_{r,K}.
\end{aligned}
\end{equation}

\paragraph{An alternative viewpoint: reducing communication.} Another way of viewing the problem is as follows: consider as a baseline processing $T$ stochastic gradient on each machine, but communicating at every step, i.e.~$T$ times, thus implementing $T$ steps of SGD, using a mini-batch of $M$ samples (one from each machine).   Can we achieve the same performance as this baseline while communicating less frequently?  Communicating only $R$ times instead of $T$ precisely brings us back to the model we are considering, and all the methods discussed above ($R$ steps of MB-SGD with mini-batches of size $KM=TM/R$, or Local SGD with $T$ steps per machine and $R$ averaging steps) reduce the communication.  Checking that $R<T$ rounds achieve the same accuracy as the dense communication baseline is a starting point, but the question is how small can we push $R$ (while keeping $T=KR$ fixed) before accuracy degrades. Better error guarantees (in terms of $K,M$ and $R$) mean we can use a smaller $R$ with less degradation, and the smallest $R$ with no asymptotic degradation can be directly calculated from the error guarantee \citep[see, e.g., discussion in][]{cotter2011better}.

In order to prove convergence guarantees, we rely on several assumptions about the problem.  Central to our analysis will be a parameter $\sdiff^2$ which, in some sense, describes the level of heterogeneity in the problem.  It is possible to analyze heterogeneous distributed optimization without any bound on the relatedness of the local objective---this is the typical setup in the consensus optimization literature \cite[e.g.][]{boyd2011distributed,nedic2009distributed,nedic2010constrained,ram2010distributed} and indeed our analysis of Minibatch SGD does not rely on this parameter.  Nevertheless, such an assumption \emph{is} required by the existing analyses of Local SGD \cite{khaled2019better,koloskova2020unified,karimireddy2019scaffold} which we would like to compare to, and, as we will show, is in fact necessary for the convergence of Local SGD. Following prior work \cite{koloskova2020unified,karimireddy2019scaffold}, we define 
\begin{equation}\label{eq:heterogeneity-parameter}
\sdiff^2 = \frac{1}{M}\sum_{m=1}^M \nrm*{\nabla F_m(x^*)}^2
\end{equation}
Since $\frac{1}{M}\sum_{m=1}^M \nabla F_m(x^*) = \nabla F(x^*) = 0$, this captures, in some sense, the variation in the local gradients {\em at the optimum}. When $\sdiff^2 = 0$, all $F_m$ share at least one minimizer, and when $\sdiff^2$ is large, there is great disagreement between the local objectives. While homogeneous objectives (i.e.~$F_m=F$) have $\sdiff^2 = 0$, the converse is not true! Even when $\sdiff^2 = 0$, $\nabla F_m(x)$ might be different than $\nabla F_{n}(x)$ for $x \neq x^*$.

Throughout, we assume that $f(\cdot;z)$ is {\bf $H$-smooth} for all $z$ meaning
\begin{equation} \label{eq:def-smooth}
    f(y;z) \leq f(x;z) + \inner{\nabla f(x;z)}{y-x} + \frac{H}{2}\nrm{x-y}^2\qquad\forall_{x,y,z},
\end{equation}
We assume the {\bf variance of the stochastic gradients} on each machine is bounded, either uniformly
\begin{equation}\label{eq:def-sigma}
\E_{z^m\sim\mc{D}^m}\nrm*{\nabla f(x;z^m) - \nabla F_m(x)}^2 \leq \sigma^2 \qquad\forall_{x,m}
\end{equation}
or {\bf only at the optimum} $x^*$, i.e.
\begin{equation}\label{eq:def-sigma-star}
\E_{z^m\sim\mc{D}^m}\nrm*{\nabla f(x^*;z^m) - \nabla F_m(x^*)}^2 \leq \sigma_*^2. \qquad\forall_{m}
\end{equation}

We consider guarantees of two forms:  For {\bf strongly convex local objectives}, we consider guarantees that depend on the {\bf parameter of strong convexity}, $\lambda$, for all the machines' objectives,
\begin{equation} \label{eq:def-strongly-convex}
    F_m(y) \geq F_m(x) + \inner{\nabla F_m(x)}{y-x} + \frac{\lambda}{2}\nrm{x-y}^2\qquad\forall_{x,y,m}
\end{equation}
as well as the {\bf initial sub-optimality}  $F(0) - F(x^*) \leq \Delta$, besides the smoothness, heterogeneity bound, and variance as discussed above.

We also consider guarantees for {\bf weakly convex objectives} that just rely on each local objective $F_m$ being convex (not necessarily strongly), as well as a bound on the norm of the optimum $\nrm{x^*} \leq B$, besides the smoothness, homogeneity bound, and variance as discussed above.

\section{Minibatch SGD and Accelerated Minibatch SGD}\label{sec:mbsgd}
To begin, we analyze the worst-case error of Minibatch and Accelerated Minibatch SGD in the heterogeneous setting. A simple observation is that, despite the heterogeneity of the objective, the minibatch gradients $g_r$ \eqref{eq:mbsgd-updates} are unbiased estimates of $\nabla F$, i.e.~the overall objective's gradient:
\begin{equation}
\E g_r = \E\brk*{\frac{1}{MK}\sum_{m=1}^{M}\sum_{k=1}^{K} \nabla f(x_{r}; z_{r,k}^m)} = \frac{1}{MK}\sum_{m=1}^{M}\sum_{k=1}^{K} \nabla F_m(x_{r}) = \nabla F(x_r)
\end{equation}
We are therefore updating using unbiased estimates of $\nabla F$, and can appeal to standard analysis for (accelerated) SGD.  To do so, we calculate the variance of these estimates:
\begin{align}
\E\nrm*{g_r - \nabla F(x_r)}^2 
&= \E\nrm*{\frac{1}{MK}\sum_{m=1}^{M}\sum_{k=1}^{K} \nabla f(x_{r}; z_{r,k}^m) - \nabla F(x_r)}^2 \\
&= \frac{1}{M^2K^2}\sum_{m=1}^{M}\sum_{k=1}^{K}\E\nrm*{\nabla f(x_{r}; z_{r,k}^m) - \nabla F_m(x_r)}^2 
\leq \frac{\sigma^2}{MK}
\end{align}
Interestingly, the variance is always reduced by $MK$ and is not effected by the heterogeneity $\sdiff$.  Plugging this calculation\footnote{A similar calculation establishes the variance at $x^*$ is $\sigma^2_*/MK$.} into the analysis of SGD (see details in Appendix \ref{app:mbsgd-upper-bound-proofs}) yields:
\begin{restatable}{theorem}{mbsgdUB}\label{thm:mbsgd-upper-bound}
A weighted average of the Minibatch SGD iterates satisfies for a universal constant $c$
\begin{align*}
\E F(\hat{x}) - F^* &\leq c\cdot\prn*{\frac{HB^2}{R} + \frac{\sigma_* B}{\sqrt{MKR}}} &&\textrm{under the convex assumptions,} \\
\E F(\hat{x}) - F^* &\leq c\cdot\prn*{\frac{H\Delta}{\lambda}\exp\prn*{-\frac{\lambda R}{2H}} + \frac{\sigma_*^2}{\lambda MKR}}&&\textrm{under the strongly convex assumptions.}
\end{align*}
And for Accelerated Minibatch SGD\footnote{This analysis can likely also be stated in terms of $\sigma_*$, but this does not easily follow from existing work on accelerated SGD.} it guarantees
\begin{align*}
\E F(\hat{x}) - F^* &\leq c\cdot\prn*{\frac{HB^2}{R^2} + \frac{\sigma B}{\sqrt{MKR}}} &&\textrm{under the convex assumptions,}\\
\E F(\hat{x}) - F^* &\leq c\cdot\prn*{\Delta\exp\prn*{-\frac{\sqrt{\lambda}R}{c_3\sqrt{H}}} + \frac{\sigma^2}{\lambda MKR}}&&\textrm{under the strongly convex assumptions.}
\end{align*}
\end{restatable}


The most important feature of these guarantees is that they are completely independent of $\sdiff^2$. Since these upper bounds are known to be tight in the homogeneous case (where $\sdiff^2 = 0$) \cite{nemirovskyyudin1983,nesterov2004introductory}, this means that both algorithms are essentially immune to the heterogeneity of the problem, performing equally well for homogeneous objectives as they do for arbitrarily heterogeneous ones. 

\section{Local SGD for heterogeneous data}\label{sec:local-sgd}
Recently, \citet{khaled2019better} and \citet{koloskova2020unified} analyzed Local SGD for the heterogeneous data setting. Their guarantees are summarized in Tables \ref{tab:prior-local-analysis-convex} and \ref{tab:prior-local-analysis-strongly-convex} along with the analysis of (Accelerated) Minibatch SGD.
Also included is the table is a guarantee for \textsc{SCAFFOLD}\footnote{
\citet{karimireddy2019scaffold} also analyzed a variant of heterogeneous Local SGD (refereed to as \textsc{FedAvg}), that, as we discuss in Appendix  \ref{app:scaffold}, ends up essentially equivalent to Minibatch SGD. The guarantee is thus not on the performance of Local SGD as in \eqref{eq:localsgd-updates}, but rather a loose upper bound for Minibatch SGD, and so we do not include it in the Tables.  \citeauthor{karimireddy2019scaffold} consider the more general framework where only a random subset of the machines are used in each round---in the Tables here we present the analysis as it applies to our setting where all machines are used in each round.  See Appendix \ref{app:scaffold} for more details.}
\citep{karimireddy2019scaffold}, a related method for heterogeneous distributed optimization.

Upon inspection, among the previously published upper bounds for heterogeneous Local SGD (that we are aware of), \citeauthor{koloskova2020unified}'s is the tightest, and dominates the others.  However, even this guarantee is the sum of the Minibatch SGD bound plus additional terms, and is thus worse in every regime, and cannot show improvement over Minibatch SGD. But does this reflect a weakness of their analysis, or a true weakness of Local SGD? Indeed, \citet{woodworth2020local} showed a tighter upper bound for Local SGD {\em in the homogeneous case}, which improves over \citet{koloskova2020unified} (for $\sdiff=0$) and {\em does} show improvement over Minibatch SGD when communication is infrequent.  But can we generalize \citeauthor{woodworth2020local}'s bound also to the heterogeneous case?  
Optimistically, we might hope that the $(\sdiff / R)^{2/3}$ dependence on heterogeneity in these bounds could be improved. After all, Minibatch SGD's rate is independent of $\sdiff^2$, so perhaps Local SGD's could be, too? 

Unfortunately, it is already known that some dependence on $\sdiff$ is necessary, as \citet{karimireddy2019scaffold} have shown a lower bound\footnote{As stated by \citet[Theorem II]{karimireddy2019scaffold}, the lower bound is for their \textsc{FedAvg} method, which is the same as \eqref{eq:inner-outer-updates} in Section \ref{sec:inner-outer}. However, their lower bound should be qualified, since with an optimal choice of stepsize parameters the $\sdiff$-dependence can be avoided and the lower bound does not hold (see Section \ref{sec:inner-outer} and Appendix \ref{app:scaffold}).  The more accurate statement is that their lower bound is for ``traditional'' Local SGD, i.e.~when $\eta_{\textrm{inner}}=\eta_{\textrm{outer}}$ in the notation of Section \ref{sec:inner-outer}, or $\eta_g=1$ in \citeauthor{karimireddy2019scaffold}'s notation.} of $\sdiff^2 / (\lambda R^2)$ in the strongly convex case, which suggests a lower bound of $\sdiff B / R$ in the weakly convex case. But perhaps the  \citeauthor{koloskova2020unified} analysis can be improved to match this bound? If the $\sdiff B / R$ term from \citeauthor{karimireddy2019scaffold}'s lower bound were possible, it would be lower order than $HB^2/R$ for $\sdiff < HB$, and we would see no slow-down until the level of heterogeneity is fairly large. On the other hand, if the $(\sdiff / R)^{2/3}$ term from \citeauthor{koloskova2020unified} cannot be improved, then we see a slowdown as soon as $\sdiff = \Omega(HB/R)$, i.e.~even for very small $\sdiff$! Which is the correct rate here?

We now show that the poor dependence on $\sdiff$ from the \citeauthor{koloskova2020unified} analysis cannot be improved. Consequently, for sufficiently heterogeneous data, Local SGD is strictly worse than Minibatch SGD, regardless of the frequency of communication, unless the level of heterogeneity is very small.
\begin{restatable}{theorem}{localSGDlowerbound}\label{thm:local-sgd-lower-bound}
For any $M$, $K$, and $R$ there exist objectives in four dimensions such that Local SGD initialized at zero and using any fixed stepsize $\eta$ will have suboptimality at least
\begin{align*}
\E F(\hat{x}) - F^* &\geq c\cdot\prn*{\min\crl*{\frac{HB^2}{R},\, \frac{\prn*{H\sdiff^2 B^4}^{1/3}}{R^{2/3}}} 
+ \frac{\prn*{H\sigma^2B^4}^{1/3}}{K^{2/3}R^{2/3}} + \frac{\sigma B}{\sqrt{MKR}}} \\
\E F(\hat{x}) - F^* &\geq c\cdot\prn*{\min\crl*{\Delta\exp\prn*{-\frac{6\lambda R}{H}},\, \frac{H\sdiff^2}{\lambda^2R^2}} + \min\crl*{\Delta,\,\frac{H\sigma^2}{\lambda^2 K^2R^2}} + \frac{\sigma^2}{\lambda MKR}}
\end{align*}
under the convex and strongly convex assumptions (for $H \geq 16\lambda$), respectively.
\end{restatable}

This is proven in Appendix \ref{app:local-sgd-lower-bound} using a similar approach to a lower bound for the homogeneous case from \citet{woodworth2020local}, and it is conceptually similar to the lower bounds for heterogeneous objectives of \citet{karimireddy2019scaffold}.  \citet{koloskova2020unified} also prove a lower bound, but specifically for 1-strongly convex objectives, which obscures the important role of the strong convexity parameter. 
\begin{table}
\renewcommand{\arraystretch}{1.2}
\centering
\begin{tabular}{ l l }
    \hline
    \hline 
    Method/Analysis & Worst-Case Error (i.e.~$\E F(\hat{x}) - F^* \lesssim$)  \\ 
    \hline
    \hline 
    \begin{tabular}{l} Minibatch SGD \\ {\small{Theorem \ref{thm:mbsgd-upper-bound}}} \end{tabular}
    & $\frac{HB^2}{R} + \frac{\sigma_* B}{\sqrt{MKR}}$
    \\\hline
    \begin{tabular}{l} Accelerated Minibatch SGD \\ {\small{Theorem \ref{thm:mbsgd-upper-bound}}} \end{tabular}
    & $\frac{HB^2}{R^2} + \frac{\sigma B}{\sqrt{MKR}}$
    \\\hline
    \begin{tabular}{l} Local SGD \\ {\small\citet{koloskova2020unified}} \end{tabular}
    & $\frac{HB^2}{R} + \frac{\sigma_* B}{\sqrt{MKR}} + \frac{\prn*{H\sdiff^2B^4}^{1/3}}{R^{2/3}} + \frac{\prn*{H\sigma_*^2B^4}^{1/3}}{K^{1/3}R^{2/3}} $
    \\\hline
    \begin{tabular}{l} Local SGD \\ {\small\citet{khaled2019better}} \end{tabular}
    & $\frac{HB^2}{R} + \frac{B\sqrt{\sigma_*^2 + \sdiff^2} }{\sqrt{MKR}} + \frac{\prn*{H(\sigma_*^2 + \sdiff^2)B^4}^{1/3}}{R^{2/3}} $
    \\\hline
    \begin{tabular}{l} \textsc{SCAFFOLD} \\ \citet{karimireddy2019scaffold} \end{tabular}
    & $\frac{HB^2}{R} + \frac{\sigma B}{\sqrt{MKR}} + 
    \frac{\sdiff^2}{HR} + \frac{\sigma \sdiff}{H\sqrt{MKR}}$
    \\\hline
    \begin{tabular}{l} Local SGD \\ {\small Theorem \ref{thm:local-sgd-uppper-bound}} \end{tabular}
    & $\frac{HB^2}{\mathbf{K}R} + \frac{\sigma_* B}{\sqrt{MKR}} + \frac{\prn*{H\bar{\zeta}^2B^4}^{1/3}}{R^{2/3}} + \frac{\prn*{H\sigma^2B^4}^{1/3}}{K^{1/3}R^{2/3}}$
    \\\hline
    \begin{tabular}{l} Local SGD Lower Bound \\ {\small{Theorem \ref{thm:local-sgd-lower-bound}}} \end{tabular}
    & $\min\crl*{\frac{HB^2}{R},\, \frac{\prn*{H\sdiff^2 B^4}^{1/3}}{R^{2/3}}} + \frac{\sigma B}{\sqrt{MKR}} + \frac{\prn*{H\sigma^2B^4}^{1/3}}{K^{2/3}R^{2/3}}$
    \\\hline
    \begin{tabular}{l} Algorithm-Independent Lower Bound \\ {\small{Theorem \ref{thm:dzr-lower-bound}}} \end{tabular}
    & $\min\crl*{\frac{HB^2}{R^2},\, \frac{\sdiff^2}{HR^2}} + \frac{\sigma B}{\sqrt{MKR}}$
    \\\hline
    \hline 
\end{tabular}
\caption{Guarantees under the convex assumptions. See \eqref{eq:def-zetabar} for a definition and discussion of $\bar{\zeta}$.
\label{tab:prior-local-analysis-convex}}
\end{table}

\begin{table}
\renewcommand{\arraystretch}{1.2}
\centering
\begin{tabular}{ l l }
    \hline 
    \hline 
    Method/Analysis & Worst-Case Error (i.e.~$\E F(\hat{x}) - F^* \lesssim$)  \\ 
    \hline 
    \hline 
    \begin{tabular}{l} Minibatch SGD \\ {\small{Theorem \ref{thm:mbsgd-upper-bound}}} \end{tabular}
    & $\frac{H\Delta}{\lambda}\exp\prn*{\frac{-\lambda R}{H}} + \frac{\sigma_*^2}{\lambda MKR}$
    \\\hline
    \begin{tabular}{l} Accelerated Minibatch SGD \\ {\small{Theorem \ref{thm:mbsgd-upper-bound}}} \end{tabular}
    & $\Delta\exp\prn*{\frac{-\sqrt{\lambda}R}{\sqrt{H}}} + \frac{\sigma^2}{\lambda MKR}$
    \\\hline
    \begin{tabular}{l} Local SGD \\ {\small\citet{koloskova2020unified}} \end{tabular}
    & $\frac{\sigma_*^2}{\lambda MKR} + \frac{H\sdiff^2}{\lambda^2 R^2} + \frac{H\sigma_*^2}{\lambda^2 KR^2}$ $\quad$ for $R \geq \tilde{\Omega}\prn*{\frac{H}{\lambda}\log\Delta}$
    \\\hline
    \begin{tabular}{l} \textsc{SCAFFOLD} \\ \citet{karimireddy2019scaffold} \end{tabular}
    & $\prn*{H\Delta + \frac{\lambda\sdiff^2}{H^2}} \exp\prn*{\frac{-\lambda R}{H}} + \frac{\sigma^2}{\lambda MKR}$
    \\\hline
    \begin{tabular}{l} Local SGD \\ {\small Theorem \ref{thm:local-sgd-uppper-bound}} \end{tabular}
    & $\frac{HB^2}{HKR+\lambda K^2R^2} + \frac{\sigma_* B}{\sqrt{MKR}} + \frac{H\bar{\zeta}^2}{\lambda^2R^2} + \frac{H\sigma^2}{\lambda^2 KR^2}$
    \\\hline
    \begin{tabular}{l} Local SGD Lower Bound \\ {\small{Theorem \ref{thm:local-sgd-lower-bound}}} \end{tabular}
    & $\min\crl*{\Delta\exp\prn*{\frac{-\lambda R}{H}},\, \frac{H\sdiff^2}{\lambda^2R^2}} + \frac{\sigma^2}{\lambda MKR} + \min\crl*{\Delta,\,\frac{H\sigma^2}{\lambda^2 K^2R^2}} $
    \\\hline
    \begin{tabular}{l} Algorithm-Independent \\ Lower Bound {\small{Theorem \ref{thm:local-sgd-lower-bound}}} \end{tabular}
    & $\min\crl*{\frac{\Delta \sqrt{\lambda}}{\sqrt{H}},\, \frac{\lambda\sdiff^2}{H^2} }\exp\prn*{-\frac{\sqrt{\lambda}R}{\sqrt{H}}} + \frac{\sigma^2}{\lambda MKR}$
    \\\hline
    \hline
\end{tabular}
\caption{Guarantees under the strongly convex assumptions, with log factors omitted. See \eqref{eq:def-zetabar} for a definition and discussion of $\bar{\zeta}$. \label{tab:prior-local-analysis-strongly-convex}}
\end{table}

In the convex case, this lower bound closely resembles the upper bound of \citet{koloskova2020unified}. Focusing on the case $H = B = \sigma^2 = 1$ to emphasize the role of $\sdiff^2$, the only gaps are between \textbf{(i)} a term $1/(K^{1/3}R^{2/3})$ vs $1/(K^{2/3}R^{2/3})$---a gap which also exists in the homogeneous case \cite{woodworth2020local}---and \textbf{(ii)} another term $\max\cb{1/R,(\sdiff / R)^{2/3}}$ vs $\min\cb{1/R,\ (\sdiff / R)^{2/3}}$. Consequently, for $\sdiff^2 \geq 1/R \Rightarrow (\sdiff / R)^{2/3} \geq 1/R$, the lower bound shows that Local SGD has error at least $1/R + 1/\sqrt{MKR}$ and thus performs strictly worse than Minibatch, regardless of $K$. This is quite surprising---Local SGD is often suggested as an improvement over Minibatch SGD for the heterogeneous setting, yet we see that even a small degree of heterogeneity can make it much worse. Furthermore, increasing the duration of each round, $K$, is often thought of as more beneficial for Local SGD than Minibatch SGD, but the lower bound indicates it does little to help Local SGD in the heterogeneous setting. 

Similarly, in the strongly convex case, the lower bound from Theorem \ref{thm:local-sgd-lower-bound} nearly matches the upper bound of \citeauthor{koloskova2020unified}.
Focusing on the case $H=B=\sigma=1$ in order to emphasize the role of $\sdiff$, the only differences are between a term $1/(KR^2)$ versus $1/(K^2R^2)$---a gap which also exists in the homogeneous case \cite{woodworth2020local}---and between $\exp(- \lambda R) + \sdiff^2/(\lambda^2R^2)$ and $\min\crl*{\exp(- \lambda R),\,\sdiff^2/(\lambda^2R^2)}$. The latter gap is somewhat more substantial than the convex case, but nevertheless indicates that the $\sdiff^2 / (\lambda^2R^2)$ rate cannot be improved until the number of rounds of communication is at least the condition number (or $\sdiff^2$ is very large).


Thus, the lower bounds indicate that it is not possible to radically improve over the \citeauthor{koloskova2020unified} analysis, and thus over Minibatch SGD for even moderate heterogeneity, without stronger assumptions.

In order to obtain an improvement over Minibatch SGD in a heterogeneous setting, at least with very low heterogeneity, we introduce a modification to the heterogeneity measure $\sdiff$ which bounds the difference between the local objectives' gradients everywhere (beyond just at $x^*$):
\begin{equation}\label{eq:def-zetabar}
\sup_x \max_{1\leq m \leq M} \nrm*{\nabla F_m(x) - \nabla F(x)}^2 \leq \bar{\zeta}^2
\end{equation}
This quantity precisely captures homogeneity since $\bar{\zeta}^2 = 0$ if and only if $F_m = F$ (up to an irrelevant additive constant). In terms of $\bar{\zeta}^2$, we are able to analyze Local SGD and see a smooth transition from the heterogeneous ($\bar{\zeta}^2$ large) to homogeneous ($\bar{\zeta}^2=0$) setting. 
\begin{restatable}{theorem}{localsgdupperbound}{\label{thm:local-sgd-uppper-bound}}
With the additional condition \eqref{eq:def-zetabar}, an average of the Local SGD iterates guarantees under the convex and strongly convex assumptions, respectively
\begin{align*}
\E F(\hat{x}) - F^* &\leq 
c\cdot\prn*{\frac{HB^2}{KR} + \frac{\prn*{H\bar{\zeta}^2B^4}^{1/3}}{R^{2/3}} + \frac{\prn*{H\sigma^2B^4}^{1/3}}{K^{1/3}R^{2/3}} + \frac{\sigma_*B}{\sqrt{MKR}}}, \\
\E F(\hat{x}) - F^* &\leq 
c\cdot\prn*{\frac{H^2 B^2}{HKR + \lambda K^2R^2} + \prn*{\frac{H\bar{\zeta}^2}{\lambda^2 R^2} + \frac{H\sigma^2}{\lambda^2 KR^2}} \log\prn*{\frac{H}{\lambda}+KR} + \frac{\sigma_*^2}{\lambda MKR}}.
\end{align*}
\end{restatable}
We prove the Theorem in Appendix \ref{app:zeta-everywhere-upper-bound}.  This is the first analysis of Local SGD, or any other method for heterogeneous distributed optimization, which shows any improvement over Minibatch SGD in any heterogeneous regime\footnote{\citet{karimireddy2019scaffold} establish a guarantee for  \textsc{SCAFFOLD} that improves over Minibatch SGD in a setting where only a random subset of the machines are available in each iteration and $\sdiff$ is sufficiently small. Here we refer to the distributed optimization setting of Section \ref{sec:setup}, where all machines are used in each iteration.}.  When $\bar{\zeta}=0$, Theorem \ref{thm:local-sgd-uppper-bound} reduces to the homogeneous analysis of Local SGD of \citet[Theorem 2]{woodworth2020local}, which already showed that in that case, we see improvement when $K \gtrsim R$.  Theorem \ref{thm:local-sgd-uppper-bound} degrades smoothly when $\bar{\zeta}>0$, and shows improvement for Local SGD over Minibatch SGD also when $\bar{\zeta}^2 \lesssim 1/R$ in the convex case, 
i.e.~with low, yet positive, heterogeneity. It is yet unclear whether this rate of convergence can be ensured in terms of $\sdiff^2$ rather than $\bar{\zeta}^2$.

\paragraph{Experimental evidence}
Finally, while Theorem \ref{thm:local-sgd-lower-bound} proves that Local SGD is worse than Minibatch SGD unless $\sdiff^2$ is very small \emph{in the worst case}, one might hope that for ``normal'' heterogeneous problems, Local SGD might perform better than its worst case error suggests. However, a simple binary logistic regression experiment on MNIST indicates that this behavior likely extends significantly beyond the worst case. The results, depicted in Figure \ref{fig:experiments}, show that Local SGD performs worse than Minibatch SGD unless both $\sdiff$ is very small and $K$ is large.
Finally, we also observe that Minibatch SGD's performance is essentially unaffected by $\sdiff$ empirically as predicted by theory.

\begin{figure}
\centering
\includegraphics[width=0.7\textwidth]{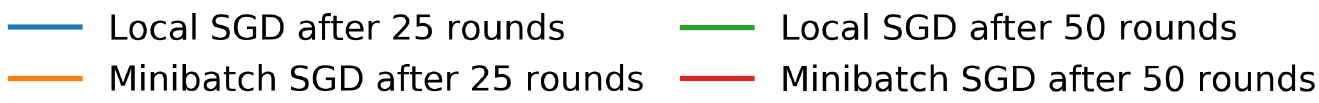}
\includegraphics[width=\textwidth]{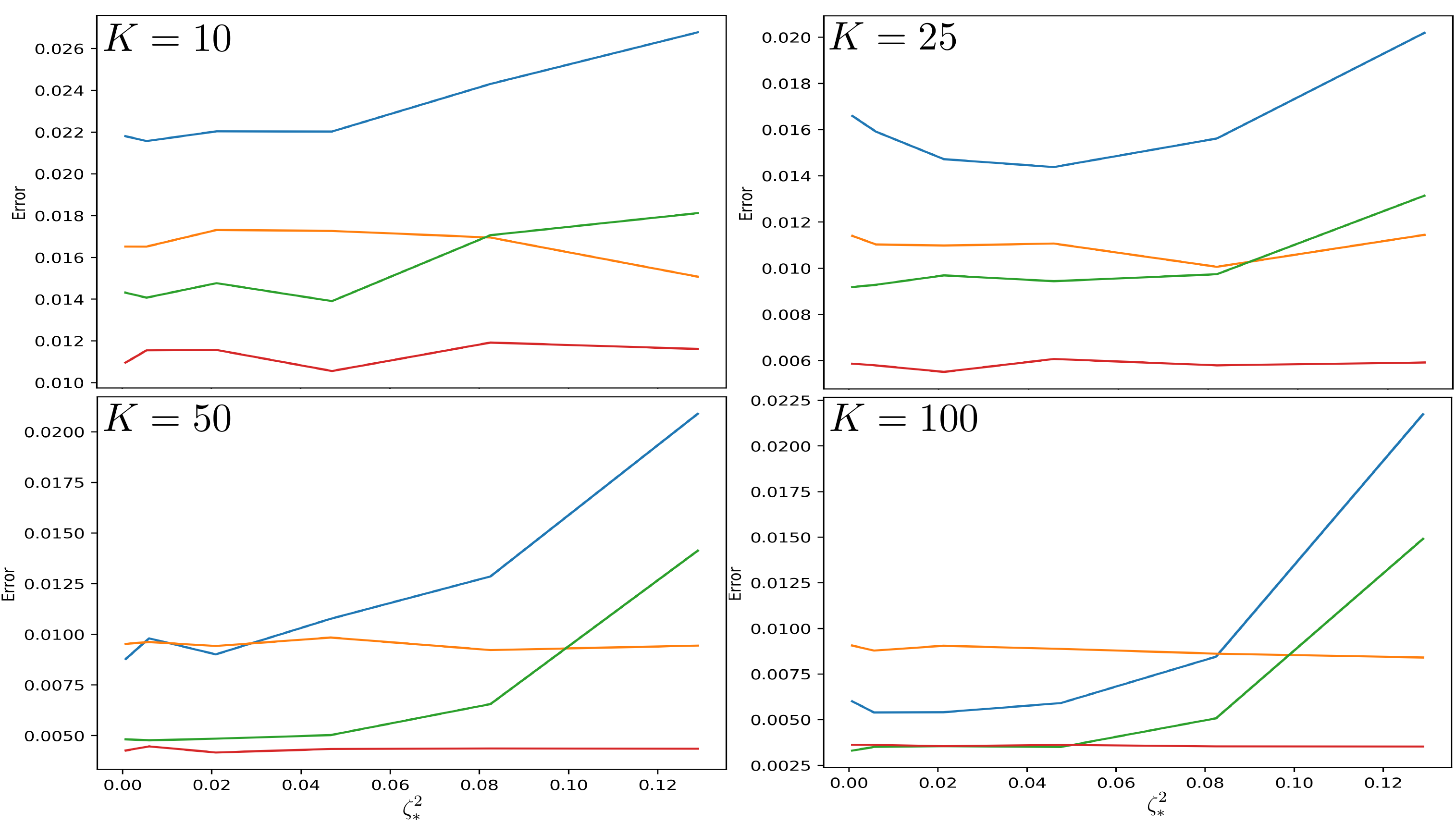}
\small \caption{\small Binary logistic regression between even vs odd digits of MNIST. Twenty-five ``tasks'' were constructed, one for each combination of $i$ vs $j$ for even $i$ and odd $j$. For $p \in \{0,20,40,60,80,100\}$, we assigned to each of $M=25$ machines $p\%$ data from task $m$, and $(100-p)\%$ data from a mixture of all tasks. For several choices of $R$ and $K$, we plot the error (averaged over four runs) versus the value of $\sdiff^2$ resulting from each choice of $p$. For both algorithms, we used the best fixed stepsize for each choice of $K$, $R$, and $\sdiff$ individually. Additional details are provided in Appendix \ref{app:experiments}.
\label{fig:experiments}}
\end{figure}

\section{Accelerated Minibatch SGD is optimal for highly heterogeneous data}\label{sec:ambsgd-optimal}

In the previous section, we showed that Local SGD is strictly worse than Minibatch SGD whenever $\sdiff^2 \geq HB/R$ (in the convex case). But perhaps a different method can improve over Accelerated Minibatch SGD? After all, Accelerated Minibatch SGD's convergence rate can only partially be improved by increasing $K$, the amount of local computation per round of communication. Even when $K \to \infty$, Accelerated Minibatch SGD is only guaranteed suboptimality $1/R^2$ or $\exp(-\lambda R)$. Can this rate be improved? Now, we show that the answer depends on the level of heterogeneity---unless $\sdiff^2$ is sufficiently small, no algorithm can improve over Accelerated Minibatch SGD, which is optimal. 
To state the lower bound, we follow \citet{carmon2017lower1} and define:
\begin{definition}[Distributed zero-respecting algorithm]
Let $\supp(v) = \crl*{i \in [d]\,:\, v_i \neq 0}$ and $\pi_m(t,m')$ be the last time before $t$ when machines $m$ and $m'$ communicated. We say that an optimization algorithm is distributed zero-respecting if the $t$th query on the $m$th machine, $x_t^m$ satisfies,
\[
\supp(x_t^m) \subseteq \bigcup_{s < t} \supp(\nabla f(x_s^m;z_s^m)) \cup \bigcup_{m' \neq m}\;\bigcup_{s \leq \pi_m(t,m')}\supp(\nabla f(x_s^{m'};z_s^{m'})).
\]
\end{definition}
That is, the coordinates of a machine's iterates are non-zero only where gradients seen by that machine were non-zero. This encompasses many first-order optimization algorithms, including Minibatch SGD, Accelerated Minibatch SGD, Local SGD, coordinate descent methods, etc. 
\begin{restatable}{theorem}{dzrlowerbound}\label{thm:dzr-lower-bound}
For any $M$, $K$, and $R$, there exist two quadratic objectives satisfying the convex and strongly convex assumptions (for $H \geq 7\lambda$) such that the output of any distributed zero-respecting algorithm will have suboptimality in the convex and strongly convex case respectively,
\begin{align*}
F(\hat{x}) - F^* &\geq c\prn*{\min\crl*{\frac{\sdiff^2}{HR^2}, \frac{HB^2}{R^2}} + \frac{\sigma B}{\sqrt{MKR}}}, \\
F(\hat{x}) - F^* &\geq c\prn*{\min\crl*{\frac{\lambda \sdiff^2}{H^2}, \frac{\Delta\sqrt{\lambda}}{\sqrt{H}}} \exp\prn*{-\frac{8R\sqrt{\lambda}}{\sqrt{H}}} + \frac{\sigma^2}{\lambda MKR}}.
\end{align*}
\end{restatable}
This is proven in Appendix \ref{app:alg-independent-lower-bound} using techniques similar to those of \citet{woodworth16tight} and \citet{arjevani2015communication}. However, care must be taken to control the parameter $\sdiff$ for the hard instance and to see how this affects the final lower bound.

\begin{restatable}{corollary}{ambsgdoptimal}\label{cor:ambsgd-optimal}
In the convex case, Accelerated Minibatch SGD is optimal when $\sdiff \geq HB$, and in the strongly convex case, it is optimal up to log factors when $\sdiff^2 \geq H^{3/2}/\sqrt{\lambda}$.
\end{restatable}
Thus, Accelerated Minibatch SGD is optimal for large $\sdiff$, but when $\sdiff$ is smaller, the lower bound does not match, and it might be possible to improve. Indeed, focusing on the convex case, there is a substantial regime $HB/\sqrt{R} \leq \sdiff \leq HB$ in which it may or may not be possible to improve over Accelerated Minibatch SGD. Is it possible to improve in this regime, and if so, with what algorithm? From the previous section, we know that Local SGD is certainly \emph{not} such an algorithm.
Thus, an important question for future work is \textbf{what can be done when $\mathbf{\sdiff}$ is bounded, but not insignificant,} and what new algorithms may be able to improve over Accelerated Minibatch SGD in this regime?

\section{Inner and outer stepsizes}\label{sec:inner-outer}

In understanding the relationship between Minibatch SGD and Local SGD, and thinking of how to improve over them, it is useful to consider a unified method that interpolates between them.  
This involves taking SGD steps locally with one stepsize, and then when the machines communicate, they take a step in the resulting direction with a second, different stepsize.  Such a dual-stepsize approach was already presented and analyzed as \textsc{FedAvg} by \citet{karimireddy2019scaffold}.
We will refer to these two stepsizes as ``inner'' and ``outer'' stepsizes, respectively, and consider\removed{\footnote{The parametrization of the step-sizes here is different than in \citet{karimireddy2019scaffold} in order to emphasize the connection to Minibatch SGD. \citeauthor{karimireddy2019scaffold}'s \textsc{FedAvg} is equivalent to \eqref{eq:inner-outer-updates} with their $\eta_l=\eta_{\textrm{inner}}$ and their $\eta_g = \eta_{\textrm{outer}}/\eta_{\textrm{inner}}$.  Our presentation allows $\eta_{\textrm{inner}}=0$, while in terms of their $\eta_g,\eta_l$, one can only approach this as $\eta_l\rightarrow 0$.}}
\begin{equation}\label{eq:inner-outer-updates}
\begin{aligned}
x_{r,k}^m &= x_{r,k-1}^m - \eta_{\textrm{inner}}\nabla f(x_{r,k-1}^m; z_{r,k-1}^m) &\quad\forall_{m\in [M], k \in [K]}\\
x_{r+1,0}^m &= x_{r,0}^m - \eta_{\textrm{outer}}\frac{1}{M}\sum_{n=1}^M \sum_{k=1}^{K}\nabla f(x_{r,k-1}^m; z_{r,k-1}^m)  &\quad\forall_{m\in [M], r \in [R]}
\end{aligned}
\end{equation}
Choosing $\eta_{\textrm{inner}} = 0$, this is equivalent to Minibatch SGD with stepsize $\eta_{\textrm{outer}}$, and choosing $\eta_{\textrm{inner}} = \eta_{\textrm{outer}}$ recovers Local SGD. Therefore, when the stepsizes are chosen optimally, this algorithm is always at least as good as both Minibatch and Local SGD. I.e.~we have that using \eqref{eq:inner-outer-updates} with optimal stepsizes, under the weakly convex assumptions:
\begin{equation}\label{eq:inner-outer-bound}
\begin{aligned}
\E F(\hat{x}) - F^* \leq \min \bigg\{ &O\left( \frac{HB^2}{R} + \frac{\sigma_* B}{\sqrt{MKR}}\right),  \\
& O\left( \frac{HB^2}{KR} + \frac{\prn*{H\bar{\zeta}^2B^4}^{1/3}}{R^{2/3}} + \frac{\prn*{H\sigma^2B^4}^{1/3}}{K^{1/3}R^{2/3}} + \frac{\sigma_*B}{\sqrt{MKR}} \right) \; \bigg\}
\end{aligned}
\end{equation}
where the first option in the min is obtained by choosing $\eta_{\textrm{inner}}=0$ and the second by choosing $\eta_{\textrm{inner}}=\eta_{\textrm{outer}}$.  We can also get a similar minimum of Theorems \ref{thm:mbsgd-upper-bound} and \ref{thm:local-sgd-uppper-bound} for the strongly convex case. This analysis \eqref{eq:inner-outer-bound} improves over  \citeauthor{karimireddy2019scaffold}'s analysis of \textsc{FedAvg} as we discuss in Appendix \ref{app:scaffold}. \removed{\citeauthor{karimireddy2019scaffold} also analyze an algorithm \textsc{SCAFFOLD}, which introduces variance reduction ``control variates'' to the updates in order to match, but not improve over \eqref{eq:inner-outer-bound} in the regime $\sdiff \leq HB$, which we also discuss in Appendix \ref{app:scaffold}.} An important question is whether this method might be able to \emph{improve} over both alternatives by choosing $0 \ll \eta_{\textrm{inner}} \ll \eta_{\textrm{outer}}$. Unfortunately, existing analysis does not resolve this question.

\section{Using a Subset of Machines in Each Round}
\label{sec:subset}

An interesting modification of our problem setting, recently formulated and studied by \citet{karimireddy2019scaffold}, is a case in which only a random subset of $S \leq M$ machines are used in each round, as is the typical case in Federated Learning \cite{kairouz2019advances}.  In the homogeneous case this makes no difference, since all machines are identical anyway.  However, in a heterogeneous setting, having all machines participate at each round means that at each round we can obtain estimates of {\em all} components of the objective \eqref{eq:objective}, which is very different from only seeing a different $S<M$ components each round, while still wanting to minimize the average of all $M$ components.  

We could still consider Minibatch SGD in this setting, but in each round we would thus use a gradient estimate based on $SK$ stochastic gradients, $K$ from each of $S$ machines chosen uniformly at random (without replacement).  This is still an unbiased estimator of the overall gradient $\nabla F(x_r)$, with variance $\frac{\sigma^2}{SK}+(1-\frac{S}{M})\frac{\bar{\zeta}^2}{S}$ (and at $x^*$ the variance is $\frac{\sigma_*^2}{SK}+(1-\frac{S}{M})\frac{\sdiff^2}{S}$), resulting in the following guarantee under the convex assumptions
\begin{equation}\label{eq:S-upper-bound}
    \E F(\hat{x}) - F^* \leq O\prn*{\frac{HB^2}{R} + \frac{\sigma_* B}{\sqrt{SKR}}
    + \sqrt{1-\frac{S}{M}}\cdot \frac{\sdiff B}{\sqrt{SR}}
    },
\end{equation}
and under the strongly convex assumptions
\begin{equation}\label{eq:S-upper-bound-sc}
    \E F(\hat{x}) - F^* \leq O\prn*{\lambda B^2\exp\prn*{\frac{-\lambda R}{H}} + \frac{\sigma_*^2 }{\lambda SKR}
    + \prn*{1-\frac{S}{M}}\cdot \frac{\sdiff^2}{\lambda SR}
    }.
\end{equation}
Similar guarantees are also available for Accelerated Minibatch SGD with $\sigma$ and $\bar{\zeta}$ replacing $\sigma_*$ and $\sdiff$.  The guarantees \eqref{eq:S-upper-bound} and \eqref{eq:S-upper-bound-sc} are valid also for Minibatch SGD (i.e.~with $\eta_\textrm{inner}=0$) and thus for the inner/outer stepsize updates \eqref{eq:inner-outer-updates}. These guarantees improve over what was shown for the inner/outer stepsize variant by \citeauthor{karimireddy2019scaffold}, but \citeauthor{karimireddy2019scaffold} also presents \textsc{SCAFFOLD} which shows benefits over Minibatch SGD in some regimes under this setting (see discussion in Appendix \ref{app:scaffold})

\section{Discussion}

Despite extensive efforts to analyze Local SGD both in homogeneous and heterogeneous settings, nearly all efforts have failed to show that Local SGD improves over Minibatch SGD in any situation. In fact, as far as we are aware, only \citet{woodworth2020local} were able to show any improvement over Minibatch SGD in any regime, and their work is confined to the easier homogeneous setting. This raised the question of what happens in the heterogeneous case, which is substantially more difficult than the homogeneous setting, and where it is plausibly necessary to deviate from the quite na\"ive approach of (Accelerated) Minibatch SGD. In this paper, we conducted a careful analysis and comparison of Local and Minibatch SGD in the heterogeneous case and showed that moving from the homogeneous to heterogeneous setting significantly harms the performance of Local SGD.
For instance, in the convex case  Local SGD is strictly worse than Minibatch SGD unless $\sdiff^2 < H^2B^2/R$.  This indicates that any benefits of Local over Minibatch SGD actually lie in the homogeneous or near-homogenous setting, and not in highly hetrogenous settings (unless additional assumptions are made). We also show similar benefits for Local SGD as those shown by \citeauthor{woodworth2020local}, extending them to a near-homogeneous regime, whenever $\bar{\zeta}^2$ is bounded and less than $1/R$.  This is the first analysis that shows any benefit for Local SGD over Minibatch SGD in any non-homogeneous regime.  At the other extreme, with high heterogeneity $\sdiff > HB$, we show that Accelerated Mini-Batch SGD is optimal.  But this leaves open the regime of medium heterogeneity, of $HB/\sqrt{R} \ll \sdiff \ll HB$, where it might be possible to improve over (Accelerated) Minibatch SGD, but {\em not} using Local SGD---can other methods be devised for this regime?

\paragraph{Acknowledgements}
This work is partially supported by NSF/BSF award 1718970, NSF-DMS 1547396, and a Google Faculty Research Award. BW is supported by a Google PhD Fellowship.

\bibliographystyle{plainnat}
\bibliography{bibliography}

\appendix

\section{Proof of Theorem \ref{thm:mbsgd-upper-bound}}\label{app:mbsgd-upper-bound-proofs}
In this Appendix, we prove Theorem \ref{thm:mbsgd-upper-bound}, starting with an analysis of Minibatch SGD, and proceeding to analyze Accelerated Minibatch SGD.
\subsection{Minibatch SGD for heterogeneous objectives}\label{sec:mbsgd_rates} 

For the proof, we will use the following standard property of convex functions:
\begin{lemma}[Co-Coercivity of the Gradient]\label{lem:co-coercivity}
For any $H$-smooth and convex $F$, and any $x$, and $y$,
\[
\nrm*{\nabla F(x) - \nabla F(y)}^2 \leq H\inner{\nabla F(x) - \nabla F(y)}{x - y},
\]
and
\[
\nrm*{\nabla F(x) - \nabla F(y)}^2 \leq 2H\prn*{F(x) - F(y) - \inner{\nabla F(y)}{x-y}}.
\]
\end{lemma}

Also note the following result from \citet{stich2019unified}, which is useful for optimizing the step-sizes.   
\begin{lemma}[\citet{stich2019unified}, Lemma 3]\label{lem:sc-exp-rate}
Consider non-negative sequences$\{r_t\}_{t\geq 0}$ and $\{s_t\}_{t\geq 0}$, which satisfy:
\begin{align}
    r_{t+1} \leq (1-a\eta_t)r_t - b \eta_ts_t + c\eta_t^2,
\end{align}
for non-negative step-sizes $\eta_t\leq \frac{1}{d}, \forall t$, for a parameter $d\geq a$, $d>0$. Then there exist a choice of step-sizes $\eta_t$ and averaging weights $w_t$, such that:
\begin{align}
    \frac{b}{W_T} \sum_{t=0}^{T} sw_t + ar_{T+1} \leq 32dr_0\exp\sb{-\frac{aT}{2d}} + \frac{36c}{aT},
\end{align}
for $W_T:=\sum_{t=0}^{T} w_t$.
\end{lemma}

Finally we can prove the following result for Minibatch SGD in this setting. 
\begin{theorem} 
Under the convex assumptions,
the average of the iterates of Minibatch SGD guarantees for a universal constant $c$,
\[
\E F(\hat{x}) - F^* \leq c\cdot\frac{H B^2}{R} + c\cdot\frac{\sigma_* B}{\sqrt{MKR}}.
\]
Under the strongly convex assumptions, a weighted average of its iterates guarantees
\[
\E F(\hat{x}) - F^* \leq c\cdot \frac{H\Delta}{\lambda}\exp\sb{-\frac{\lambda R}{8H}} + c\cdot\frac{\sigma_*^2}{\lambda MKR}.
\]
\end{theorem}
\begin{proof}
By the $\lambda$-strong convexity of $F$, where $\lambda$ might be equal to zero:
\begin{align}
\E\nrm*{x_{t+1} - x^*}^2 
&= \E\nrm*{x_t - \eta_t\frac{1}{KM}\sum_{m=1}^M\sum_{k=1}^K\nabla f(x_t;z_t^{m,k}) - x^*}^2, \\
&= \E\nrm*{x_t - x^*}^2 - 2\eta_t\E\inner{\nabla F(x_t)}{x_t - x^*} \nonumber\\
&\qquad + \eta_t^2\E\nrm*{\frac{1}{KM}\sum_{m=1}^M\sum_{k=1}^K\nabla f(x_t;z_t^{m,k})}^2,  \\
&\leq \prn*{1-\lambda\eta_t}\E\nrm*{x_t - x^*}^2 - 2\eta_t\E\brk*{F(x_t) - F^*} \nonumber\\
&\qquad + \eta_t^2\E\nrm*{\frac{1}{KM}\sum_{m=1}^M\sum_{k=1}^K\nabla f(x_t;z_t^{m,k})}^2\label{eq:homogeneous-mbsgd-left-off}.
\end{align}
By the $H$-smoothness of $f(\cdot\,;z)$, we can bound the final term with
\begin{align}
&\E\nrm*{\frac{1}{KM}\sum_{m=1}^M\sum_{k=1}^K\nabla f(x_t;z_t^{m,k})}^2 \nonumber\\
&= \E\nrm*{\frac{1}{KM}\sum_{m=1}^M\sum_{k=1}^K\brk*{\nabla f(x_t;z_t^{m,k}) - \nabla f(x^*; z_t^{m,k}) + \nabla f(x^*;z_t^{m,k})}}^2, \\
&\leq 2\E\nrm*{\frac{1}{KM}\sum_{m=1}^M\sum_{k=1}^K\brk*{\nabla f(x_t;z_t^{m,k}) - \nabla f(x^*; z_t^{m,k})}}^2 \nonumber\\
&\qquad + 2\E\nrm*{\frac{1}{KM}\sum_{m=1}^M\sum_{k=1}^K\nabla f(x^*;z_t^{m,k})}^2, \\
&\leq \frac{2}{KM}\sum_{m=1}^M\sum_{k=1}^K\E\nrm*{\nabla f(x_t;z_t^{m,k}) - \nabla f(x^*; z_t^{m,k})}^2 \nonumber\\
&\qquad + 2\E\nrm*{\frac{1}{KM}\sum_{m=1}^M\sum_{k=1}^K\nabla f(x^*;z_t^{m,k}) - \nabla F_m(x^*)}^2, \\
&\leq \frac{4H}{KM}\sum_{m=1}^M\sum_{k=1}^K\E\brk*{f(x_t;z_t^{m,k}) - f(x^*;z_t^{m,k}) - \inner{\nabla f(x^*;z_t^{m,k})}{x_t - x^*}} + \frac{2\sigma_*^2}{MK}, \\
&= 4H\E\brk*{F(x_t) - F^*} + \frac{2\sigma_*^2}{MK}.
\end{align}
Where, for the third inequality we used Lemma \ref{lem:co-coercivity}.
Plugging this back into \eqref{eq:homogeneous-mbsgd-left-off}, then for $\eta_t \leq \frac{1}{4H}$,
\begin{align}
\E\nrm*{x_{t+1} - x^*}^2 
&\leq \prn*{1-\lambda\eta_t}\E\nrm*{x_t - x^*}^2 - 2\eta_t\prn*{1 - 2H\eta_t}\E\brk*{F(x_t) - F^*} + \frac{2\eta_t^2\sigma_*^2}{MK}, \\
&\leq \prn*{1-\lambda\eta_t}\E\nrm*{x_t - x^*}^2 - \eta_t\E\brk*{F(x_t) - F^*} + \frac{2\eta_t^2\sigma_*^2}{MK} \label{eq:recursion},\\
\E\brk*{F(x_t) - F^*}
&\leq \prn*{\frac{1}{\eta_t}-\lambda}\E\nrm*{x_t - x^*}^2 - \frac{1}{\eta_t}\E\nrm*{x_{t+1} - x^*}^2 + \frac{2\eta_t\sigma_*^2}{MK}.
\end{align}
Now we look at $\lambda = 0$ and $\lambda > 0$ separately.
\paragraph{Convex case ($\lambda = 0$):}
Choose a constant step-size, 
\begin{equation}
    \eta_t = \eta = \min\crl*{\frac{1}{4H},\, \frac{B\sqrt{MK}}{\sigma_*\sqrt{T}}}.
\end{equation}
Then the averaged iterate $\bar{x}_R = \frac{1}{R}\sum_{t=1}^{R} x_t$ satisfies:
\begin{align}
\E F(\bar{x}_R) - F^* 
&\leq \frac{1}{R}\sum_{t=1}^R\E F(x_t) - F^*, \\
&\leq \frac{1}{R}\sum_{t=1}^R \frac{1}{\eta}\E\nrm*{x_t - x^*}^2 - \frac{1}{\eta}\E\nrm*{x_{t+1} - x^*}^2 + \frac{2\eta\sigma_*^2}{MK}, \\
&\leq \frac{\nrm*{x_0 - x^*}^2}{\eta R} + \frac{2\eta\sigma_*^2}{MK}, \\
&\leq \max\crl*{\frac{4HB^2}{R},\, \frac{\sigma_* B}{\sqrt{MKR}}} + \frac{2\sigma_* B}{\sqrt{MKR}}, \\
&\leq \frac{4H B^2}{R} + \frac{3\sigma_* B}{\sqrt{MKR}}.
\end{align}

\paragraph{Strongly convex case ($\lambda > 0$):}
Rewriting \eqref{eq:recursion}, 
\begin{align}
\E\nrm*{x_{t+1} - x^*}^2 & \leq \prn*{1-\lambda\eta_t}\E\nrm*{x_t - x^*}^2 - \eta_t\E\brk*{F(x_t) - F^*} + \frac{2\eta_t^2\sigma_*^2}{MK},
\end{align}
we note that it satisfies the conditions for Lemma \ref{lem:sc-exp-rate} for the specific assignment:
\begin{align}
    r_t &= \E\nrm*{x_t - x^*}^2,\ s_t = \E\brk*{F(x_t) - F^*},\\
    a &= \lambda,\ b = 1,\ c = \frac{2\sigma_*^2}{MK},\ d=4H,\ T = R.
\end{align}
Thus using Lemma \ref{lem:sc-exp-rate}, and applying Jensen's inequality we can guarantee the following convergence rate for the averaged iterate $\hat{x}_R =\frac{1}{W_R}\sum_{t=1}^{R}w_t x_t $,
\begin{align}
    \E\brk*{F(\hat{x}_R) - F^*} \leq 128H\nrm*{x_0 - x^*}^2\exp\sb{-\frac{\lambda R}{8H}} + \frac{72\sigma_*^2}{\lambda MKR},  
\end{align}
using step-size $\eta_t$ and $w_t$ given by,
\begin{align*}
    &if\ R\leq \frac{4H}{\lambda}, && \eta_t = \frac{1}{4H}, &&w_t = (1-\lambda\eta_t)^{-(t+1)},\\
    &if\ R> \frac{4H}{\lambda}\ and\ t<t_0, && \eta_t = \frac{1}{4H}, &&w_t = 0,\\
    &if\ R> \frac{4H}{\lambda}\ and\ t\geq t_0, && \eta_t = \frac{2}{\lambda(\kappa + t - t_0)}, &&w_t = (\kappa + t - t_0)^2,
\end{align*}
where $\kappa = \frac{8H}{\lambda}$ and $t_0 = \ceil{\frac{R}{2}}$. We conclude by observing that $HB^2 \leq \frac{H\Delta}{\lambda}$.
\end{proof}

\subsection{Accelerated Minibatch SGD for heterogeneous objectives}\label{sec:ac_mbsgd_rates}
We first recall some classical results from \citet{ghadimi2012optimal, ghadimi2013optimal} for accelerated variants of minibatch SGD. These results are for minimizing $F(x):=\E_{z\sim\mc{D}}f(x,z)$ where $F$ is $H$-smooth and $\lambda (\geq 0)$-strongly convex. The algorithms use unbiased stochastic gradients $\{g_t\}_{t\in [T]}$, i.e., for all $t$, $\E \sb{g_t(x)} = \nabla F(x)$ which have bounded variance for all $x$, i.e., $\E\nrm*{g_t(x) - \nabla F(x)}^2 \leq \sigma^2$. 

First consider the AC-SA algorithm \citet[c.f., Sec 3.1,][]{ghadimi2012optimal}), with step-size parameters $\{\alpha_t\}_{t\geq 1}$ and $\{\gamma_t\}_{t\geq 1}$ s.t.~$\alpha_1=1$, $\alpha_t\in(0,1)$ for any $t\geq 2$ and $\gamma_t>0$ for any $t\geq 1$. The algorithm maintains three intertwined sequences $\{x_t\}$, $\{x_t^{ag}\}$, and $\{x_t^{md}\}$, updated as follows:   
\begin{enumerate}
    \item Set the initial points $x_0^{ag}=x_0 \in X$ and $t=1$;
    \item Set $x_t^{md} = \frac{(1-\alpha_t)(\lambda+\gamma_t)}{\gamma_t + (1-\alpha^2_t)\lambda}x^{ag}_{t-1} + \frac{\alpha_t[(1-\alpha_t)\mu + \gamma_t]}{\gamma_t + (1-\alpha^2_t)\lambda}x_{t-1}$;
    \item Call the stochastic oracle to get the gradient $g_t$ at the point $x_t^{md}$;
    \item Set $x_t = \argmin_{x\in X}\left\{ \alpha_t[\inner{g_t}{x} + \frac{\lambda}{2}\nrm*{x_t^{md}-x}^2] + [(1-\alpha_t)\frac{\lambda}{2} + \frac{\gamma_t}{2}]\nrm*{x_{t-1}-x}^2\right\}$;
    \item Set $x_t^{ag} = \alpha_tx_t + (1-\alpha_t)x_{t-1}^{ag}$;
    \item Set $t\leftarrow t+1$ and go to step 1. 
\end{enumerate}

We have the following (almost optimal) convergence rate for strongly convex functions using AC-SA (see, Sec 3.1 in \cite{ghadimi2012optimal}). 
\begin{lemma} (\citet{ghadimi2012optimal}, Proposition 9)\label{lem:ac_sgd_sc}
Let $\hat{x}^{ag}$ be computed by $T$ steps of AC-SA using stochastic gradients of variance $\sigma^2$, then for a universal constant $c$,
\begin{align*}
    \E\sb{F(x^{ag}) - F(x^\star)}  \leq c\cdot \frac{H\nrm*{x_0-x^\star}^2}{T^2} + c\cdot\frac{\sigma^2}{\lambda T}.
\end{align*}
\end{lemma}
It can be adapted to the weakly convex case by noting that, if $\tilde{F}(x) := F(x) +\frac{\lambda}{2}\nrm*{x_0-x}^2$ for any $\lambda, x_0$, and $x^\star = \argmin F(x)$ then,
\begin{align*}
    &\min_y\E\sb{\tilde{F}(y)} \leq \E\sb{F(x^\star) + \frac{\lambda}{2}\nrm*{x_0-x^\star}^2},\\  
    &\Rightarrow -\E\sb{F(x^\star)} \leq -\E\sb{\min_y\tilde{F}(y)} + \frac{\lambda}{2}\nrm*{x_0-x^\star}^2,\\
    &\Rightarrow \E\sb{F(x) - F(x^\star)} \leq \E\sb{\tilde{F}(x) - \min_y \tilde{F}(y)} + \frac{\lambda}{2}\nrm*{x_0-x^\star}^2, \forall x.
\end{align*}
This also holds if we optimize the right hand side w.r.t. $\lambda$. In other words, a guarantee for the strongly convex case, can be converted to the weakly convex case, by regularizing with $\frac{\lambda}{2}\nrm*{x_0-x^\star}^2$ with optimal value of $\lambda$. This gives the following result, 
\begin{lemma}\label{lem:ac_sgd}
Let $\hat{x}^{ag}$ be computed by $T$ steps of AC-SA on the regularized objective $\tilde{F}(x) = F(x) + \frac{\sigma}{2\nrm{x_0 - x^*}\sqrt{T}}\nrm{x - x_0}^2$, where the stochastic gradients have variance $\sigma^2$, then for a constant $c$
\begin{align*}
    \E\sb{F(\hat{x}^{ag}) - F(x^\star)}  \leq c\cdot \frac{H\nrm*{x_0-x^\star}^2}{T^2} + c\cdot\frac{\sigma\nrm*{x_0-x^\star}}{\sqrt{T}}.
\end{align*}
\end{lemma}
This is minimax optimal for weakly convex functions. Next we consider the multi-stage accelerated SGD algorithm \citet[c.f., Sec 3,][]{ghadimi2013optimal} which uses the above AC-SA algorithm. Let $p_0\in X$, have bounded sub-optimality $F(p_0) - F(x^\star) \leq \Delta$, then for $k=1,2, \dots$,
\begin{enumerate}
    \item Run $N_k$ iterations of the generic AC-SA by using $x_0=p_{k-1}$, $\{\alpha_t\}_{t\geq 1}$, and $\{\gamma_t\}_{t\geq 1}$, with relevant definitions as follows,
    \begin{align*}
        N_k &=  \ceil{\max\left\{ 4\sqrt{\frac{2H}{\lambda}},\frac{128\sigma^2}{3\lambda\Delta 2^{-(k+1)}}\right\}},\\
        \alpha_t &= \frac{2}{t+1}, \gamma_t = \frac{4\phi_k}{t(t+1)},\\
        \phi_k &= \max\left\{2H, \left[\frac{\lambda\sigma^2}{3\Delta 2^{-(k-1)}N_k(N_{k}+1)(N_{k}+2)}\right]^{1/2}\right\};
    \end{align*}
    \item Set $p_k = x_{N_k}^{ag}$ where $x_{N_k}^{ag}$ is the solution obtained in the previous step.   
\end{enumerate}
We have following optimal rate for strongly convex functions for this algorithm,
\begin{lemma} (\citet{ghadimi2013optimal}, Proposition 7)\label{lem:mac_sgd}
Let $\hat{x}^{ag}$ be computed by $T$ steps of multi-stage AC-SA using stochastic gradients of variance $\sigma^2$, then for a universal constant $c$,
\begin{align*}
    \E\sb{F(\hat{x}^{ag}) - F(x^\star)}  \leq c\cdot \Delta \exp\sb{-\sqrt{\frac{\lambda}{H}}T} + c\cdot\frac{\sigma^2}{\lambda T}.
\end{align*}
\end{lemma}

\begin{theorem}
Under the convex assumptions, performing AC-SA on the regularized objective $\tilde{F}(x) = F(x) + \frac{\sigma}{2B\sqrt{MKR}}\nrm{x}^2$ guarantees for a universal constant $c$
\[
\E F(\hat{x}) - F^* \leq c\cdot\frac{HB^2}{R^2} + c\cdot\frac{\sigma B}{\sqrt{MKR}}.
\]
Under the strongly convex assumptions, the multi-stage AC-SA algorithm guarantees
\[
\E F(\hat{x}) - F^* \leq c\cdot \Delta \exp\sb{-\sqrt{\frac{\lambda}{H}}R} + c\cdot\frac{\sigma^2}{\lambda MKR}.
\]
\end{theorem}
\begin{proof}
In order to use the previous lemmas, first note that the stochastic gradient at time $t$ at point $x$ is given by $\frac{1}{MK}\sum_{m=1}^M\sum_{k=1}^K \nabla f(x;z^t_{m,k})$, where $z^t_{m,k}\sim^{i.i.d.} \mc{D}^m$ for all machines. Fortunately, its still an unbiased gradient estimate, i.e.,  $\E\sb{\frac{1}{MK}\sum_{m=1}^M\sum_{k=1}^K \nabla f(x;z^t_{m,k})} = \nabla F(x)$ since the iterates on each machine are sampled i.i.d. Its variance is given by,
\begin{align}
    &\E\nrm*{\frac{1}{MK}\sum_{m=1}^M\sum_{k=1}^K \nabla f(x;z^t_{m,k}) - \nabla F(x)}^2 \\
    &= \frac{1}{M^2K^2}\sum_{m=1}^M\sum_{k=1}^K\E\nrm*{\nabla f(x;z^t_{m,k})-\nabla F_m(x)}^2\\
    &\leq \frac{1}{M^2K^2}\sum_{m=1}^M\sum_{k=1}^K\sigma^2\\
    &= \frac{\sigma^2}{MK}
\end{align}
Plugging this into Lemmas \ref{lem:ac_sgd} and \ref{lem:mac_sgd} completes the proof.
\end{proof}

\section{Proof of Theorem
\ref{thm:local-sgd-lower-bound}}\label{app:local-sgd-lower-bound}
Consider the following function $F:\R^4\rightarrow\R$:
\begin{equation}
F(x) = \frac{1}{2}\prn*{F_1(x) + F_2(x)} = \frac{1}{2}\prn*{\E_{z^1\sim\mc{D}^1}f(x;z^1) + \E_{z^2\sim\mc{D}^2}f(x;z^2)}
\end{equation}
The distribution $z^1 \sim \mc{D}^1$ is described by $z^1 = (1, z)$ for $z \sim \mc{N}(0,\sigma^2)$. Similarly, $z^2 \sim \mc{D}^2$ is specified by $z^2 = (2, z)$ for $z \sim \mc{N}(0,\sigma^2)$. The lower bound construction will be based on just two functions. For $M > 2$ machines, we simply assign the first $\lfloor M/2 \rfloor$ machines $F_1$ and the next $\lfloor M/2 \rfloor$ machines $F_2$. This diminishes the lower bound by at most a $(M-1)/M$ factor. Therefore, we continue with the case $M=2$.

Following \citet{woodworth2020local}, we define the local functions $F_1$ and $F_2$ via the auxiliary function
\begin{gather}
g(x_1,x_2,x_3,z) = \frac{\mu}{2}\prn*{x_1 - c}^2 + \frac{H}{2}\prn*{x_2-\frac{\sqrt{\mu} c}{\sqrt{H}}}^2 + \frac{H}{8}\prn*{x_3^2 + \pp{x_3}^2} + z^\top x_3 \\
G(x_1,x_2,x_3) = \E_z g(x_1,x_2,x_3,z)
\end{gather}
where $c > 0$ and $\mu \in \brk*{\lambda, \frac{H}{16}}$ are parameters to be determined later, and where $\pp{x} := \max\crl{x,0}$.
Then, we define
\begin{gather}
f(x;(1,z)) = g(x_1,x_2,x_3,z)  + \frac{Lx_4^2}{2} + \sdiff x_4  \\
f(x;(2,z)) = g(x_1,x_2,x_3,z)  + \frac{\lambda x_4^2}{2} - \sdiff x_4 
\end{gather}
for a parameter $L \in [\lambda, H]$ to be determined later. Therefore,
\begin{gather}
F_1(x) = \E_{z^1\sim\mc{D}^1}f(x;z^1) = G(x_1,x_2,x_3)  + \frac{Lx_4^2}{2} + \sdiff x_4 \\
F_2(x) = \E_{z^2\sim\mc{D}^2}f(x;z^2) = G(x_1,x_2,x_3)  + \frac{\lambda x_4^2}{2} - \sdiff x_4 
\end{gather}
It is clear from inspection that both $F_1$ and $F_2$, and consequently $F$, are $H$-smooth and $\lambda$-strongly convex. Furthermore, the variance of the gradients is bounded by $\sigma^2$ for both $\mc{D}^1$ and $\mc{D}^2$.

It is clear that $G$ attains its minimum of zero at $\brk*{c,\frac{\sqrt{\mu} c}{\sqrt{H}},0}$ so $\nabla G\prn*{c,\frac{\sqrt{\mu} c}{\sqrt{H}},0} = 0$, and thus
\begin{equation}
\nabla F\prn*{c,\frac{\sqrt{\mu} c}{\sqrt{H}},0,0} = \nabla G\prn*{c,\frac{\sqrt{\mu} c}{\sqrt{H}},0} + \prn*{\frac{\sdiff}{2} - \frac{\sdiff}{2}}e_4 = 0
\end{equation}
From now on, we use $x^* = \brk*{c,\frac{\sqrt{\mu} c}{\sqrt{H}},0,0}$ to denote the minimizer of $F$, which has norm
\begin{equation}
\nrm{x^*}^2 = \prn*{1 + \frac{\mu}{H}}c^2 \leq 2c^2
\end{equation}
We can therefore ensure $\nrm{x^*}^2 \leq B^2$ by choosing $c^2 \leq \frac{B^2}{2}$.
Furthermore, the initial suboptimality
\begin{align}
F(0,0,0,0) - F^* = \mu c^2
\end{align}
Therefore, we can ensure $F(0,0,0,0) - F^* \leq \Delta$ by choosing $c^2 \leq \frac{\Delta}{\mu}$.
We conclude by showing that for this objective, $\sdiff^2$ bounded by
\begin{align}
\frac{1}{2}\sum_{m=1}^2 \nrm*{\nabla F_m(x^*)}^2
= \nrm*{\nabla F_2(x^*)}^2 
= \nrm*{\nabla F_1(x^*)}^2 
= \sdiff^2
\end{align}
Therefore, this objective has the desired level of heterogeneity.

Therefore, we have shown that the objective satisfies all of the necessary conditions for the lower bound. All that remains is to lower bound the error of Local SGD with a constant stepsize $\eta$ applied to this function.

\begin{lemma}\label{lem:fourth-coordinate}
For $\mu \leq 2L$, then Local SGD with any constant stepsize $\eta \leq \frac{1}{L}$ applied to $F_1$ and $F_2$ after being initialized at zero results in $\hat{x}_4$ such that
\[
\frac{(L+\mu)\hat{x}_4^2}{4} \geq 
\frac{\sdiff^2(L+\mu)}{16\mu^2}\prn*{\frac{L-\mu}{L} - \prn*{1-\mu\eta}^{K}}^2\indicatorb{\eta \leq \frac{1}{L}}\indicatorb{\prn*{1-\mu\eta}^{K} \leq \frac{L-\mu}{L}}
\] 
\end{lemma}
\begin{proof}
Since the coordinates of $F_1$ and $F_2$ are completely decoupled, the behavior of the fourth coordinate of the iterates can be analyzed separately from the others. 

Let $x_{k,r}^{(1)}$ denote the fourth coordinate of machine 1's iterate at the $k$th iteration of round $r$, and similarly for $x_{k,r}^{(2)}$. The local SGD dynamics give
\begin{gather}
x_{k+1,r}^{(1)} = x_{k,r}^{(1)} - \eta\prn*{L x_{k,r}^{(1)} + \sdiff} = - \frac{\sdiff}{L} + (1-L\eta)\prn*{x_{k,r}^{(1)} + \frac{\sdiff}{L}} \\
x_{k,r}^{(2)} = x_{k,r}^{(2)} - \eta\prn*{-\sdiff + \mu x_{k,r}^{(2)}} = \frac{\sdiff}{\mu} + \prn*{1-\mu\eta}\prn*{x_{k,r}^{(2)} - \frac{\sdiff}{\mu}}
\end{gather}
and $\hat{x}_4 = \frac{1}{2}\prn*{x_{K,R}^{(1)} +  x_{K,R}^{(2)}} = x_{0,R+1}$.
Unravelling this recursion, we have that 
\begin{equation}
x_{0,r+1} = x_{0,r+1}^{(1)} = x_{0,r+1}^{(2)} = \frac{1}{2}\prn*{\frac{\sdiff}{\mu} - \frac{\sdiff}{L} + \prn*{1-\mu\eta}^K\prn*{x_{0,r} - \frac{\sdiff}{\mu}} + \prn*{1-L\eta}^K \prn*{x_{0,r} + \frac{\sdiff}{L}}}
\end{equation}
Furthermore, if $\eta \leq \frac{1}{L}$ then $(1-L\eta) \geq 0$, so if $x_{0,r} \geq 0$ then
\begin{equation}
x_{0,r+1} \geq \frac{\sdiff}{2\mu} - \frac{\sdiff}{2L} + \prn*{1-\mu\eta}^K\prn*{\frac{x_{0,r}}{2} - \frac{\sdiff}{2\mu}} \geq \frac{\sdiff}{2\mu}\prn*{\frac{L-\mu}{L} - \prn*{1-\mu\eta}^{K}}
\end{equation}
Finally, since $x_{0,0} = 0 \geq 0$, the condition $x_{0,r} \geq 0$ will hold throughout optimization, so
\begin{equation}
\hat{x}_4 \geq \frac{\sdiff}{2\mu}\prn*{\frac{L-\mu}{L} - \prn*{1-\mu\eta}^{K}}
\end{equation}
Therefore, if $\eta \leq \frac{1}{L}$ and $\prn*{1-\mu\eta}^{K} \leq \frac{L-\mu}{L}$ then
\begin{align}
\frac{(L+\mu)\hat{x}_4^2}{4}
&\geq \frac{\sdiff^2(L+\mu)}{16\mu^2}\prn*{\frac{L-\mu}{L} - \prn*{1-\mu\eta}^{K}}^2 \label{eq:fourth-coord-lower-bound}
\end{align}
This completes the proof.
\end{proof}

We now prove the theorem:
\localSGDlowerbound*
\begin{proof}
Since the four different coordinates are completely decoupled from each other, it suffices to analyze each coordinate separately. 

In the course of proving \cite[Theorem 3][]{woodworth2020local}, \citeauthor{woodworth2020local} prove that 
\begin{multline}
\E G(\hat{x}_1, \hat{x}_2, \hat{x}_3) - G\prn*{c,\frac{\sqrt{\mu}c}{\sqrt{H}},0} \\
\geq 
\frac{\mu c^2\prn*{1-\mu\eta}^{KR}}{2} + \frac{\mu c^2}{2}\indicatorb{\eta > \frac{2}{H}} + \frac{H\eta^2\sigma^2}{18432}\indicatorb{\eta \leq \frac{2}{H}}\indicatorb{\eta \geq \frac{8}{HKR}}
\end{multline}

Furthermore, by Lemma \ref{lem:fourth-coordinate}
\begin{equation}
\frac{(L+\lambda)\hat{x}_4^2}{4} \geq 
\frac{\sdiff^2(L+\mu)}{16\mu^2}\prn*{\frac{L-\mu}{L} - \prn*{1-\mu\eta}^{K}}^2\indicatorb{\eta \leq \frac{1}{L}}\indicatorb{\prn*{1-\mu\eta}^{K} \leq \frac{L-\mu}{L}}
\end{equation}
Therefore, choosing $L = \frac{H}{2}$
\begin{align}
\E F(\hat{x}) - F^* 
&= \E G(\hat{x}_1, \hat{x}_2, \hat{x}_3) - G\prn*{c,\frac{\sqrt{\mu}c}{\sqrt{H}},0} + \frac{H+2\lambda}{8}\hat{x}_4^2 \\
&\geq \frac{\mu c^2\prn*{1-\mu\eta}^{KR}}{2} + \frac{\mu c^2}{2}\indicatorb{\eta > \frac{2}{H}} + \frac{H\eta^2\sigma^2}{18432}\indicatorb{\eta \leq \frac{2}{H}}\indicatorb{\eta \geq \frac{8}{HKR}} \nonumber\\
&\qquad+ \frac{\sdiff^2(H+2\mu)}{32\mu^2}\prn*{\frac{H-2\mu}{H} - \prn*{1-\mu\eta}^{K}}^2\indicatorb{\eta \leq \frac{2}{H}}\indicatorb{\prn*{1-\mu\eta}^{K} \leq \frac{H-2\mu}{H}}\label{eq:raw-lower-bound}
\end{align}

\subsection*{Stochastic terms}
First, we will show a lower bound in terms of $\sigma^2$ using solely the first three terms of \eqref{eq:raw-lower-bound}. Consider three cases:

\paragraph{Case 1 $\eta \geq \frac{2}{H}$:}
In this case, from the second term of \eqref{eq:raw-lower-bound} we see that 
\begin{equation}
\E F(\hat{x}) - F^* \geq \frac{\mu c^2}{2}
\end{equation}

\paragraph{Case 2 $\frac{1}{2\mu KR} \leq \eta \leq \frac{2}{H}$:}
In this case, the third term of \eqref{eq:raw-lower-bound} shows
\begin{align}
\E F(\hat{x}) - F^*
\geq  \frac{H\eta^2\sigma^2}{18432}
\end{align}
where we recalled that $\mu \leq \frac{H}{16}$, so $\eta \geq \frac{1}{2\mu KR} \geq \frac{8}{HKR}$. This is non-decreasing in $\eta$, so for any $\eta$
\begin{align}
\E F(\hat{x}) - F^* 
\geq \frac{H\sigma^2}{73728 \mu^2 K^2R^2}
\end{align}

\paragraph{Case 3 $\eta \leq \frac{2}{H}$ and $\eta \leq \frac{1}{2\mu KR}$:}
In this case, from the first term of \eqref{eq:raw-lower-bound},
\begin{align}
\E F(\hat{x}) - F^* 
\geq \frac{\mu c^2\prn*{1-\mu\eta}^{KR}}{2} 
\geq \frac{\mu c^2\prn*{1-\frac{1}{2KR}}^{KR}}{2} 
\geq \frac{\mu c^2}{4}
\end{align}

\paragraph{Combination:}
Combining these three cases, we conclude that for any $\eta$
\begin{align}
\E F(\hat{x}) - F^* 
\geq \min\crl*{\frac{\mu c^2}{2},\, \frac{H\sigma^2}{73728 \mu^2 K^2R^2},\, \frac{\mu c^2}{4}}
= \min\crl*{\frac{\mu c^2}{3},\, \frac{H\sigma^2}{73728 \mu^2 K^2R^2}}
\end{align}
This lower bound holds for any stepsize, and any $\mu \in \brk*{\lambda, \frac{H}{16}}$ and regardless of $\sdiff$. In the strongly convex case, we recall that $F(0) - F(x^*) = \mu c^2$, therefore, we choose $\mu = \lambda$, and $c^2 = \frac{\Delta}{\lambda}$ so the lower bound reads (for a universal constant $\beta$)
\begin{align}
\E F(\hat{x}) - F^* 
\geq \beta\cdot \min\crl*{\Delta,\, \frac{H\sigma^2}{\lambda^2 K^2R^2}}
\end{align}
To conclude, it is well known that any first-order method which accesses at most $MKR$ stochastic gradients with variance $\sigma^2$ for a $\lambda$-strongly convex objective will suffer error at least $\beta\frac{\sigma^2}{\lambda MKR}$ in the worst case \cite{nemirovskyyudin1983}.
Therefore, the strongly convex lower bound is
\begin{align}
\E F(\hat{x}) - F^* 
\geq \beta\cdot \min\crl*{\Delta,\, \frac{H\sigma^2}{\lambda^2 K^2R^2}} + \beta\cdot\frac{\sigma^2}{\lambda MKR}
\end{align}

In the convex case, we recall that $\nrm*{x^*}^2 \leq 2c^2$, so we choose $c^2 = \frac{B^2}{2}$, and set $\mu = \prn*{\frac{H\sigma^2}{B^2K^2R^2}}^{1/3}$ so the lower bound reads
\begin{align}
\E F(\hat{x}) - F^* 
\geq \beta\cdot \frac{\prn*{H\sigma^2B^4}}{K^{2/3}R^{2/3}}
\end{align}
To conclude, it is well known that any first-order method which accesses at most $MKR$ stochastic gradients with variance $\sigma^2$ for a convex objective with $\nrm{x^*}\leq B$ will suffer error at least $\beta\frac{\sigma B}{\sqrt{MKR}}$ in the worst case \cite{nemirovskyyudin1983}. Therefore, the convex lower bound is
\begin{align}
\E F(\hat{x}) - F^* 
\geq \beta\cdot \frac{\prn*{H\sigma^2B^4}}{K^{2/3}R^{2/3}} + \beta\cdot \frac{\sigma B}{\sqrt{MKR}}
\end{align}

\subsection*{Heterogeneity terms}
Next, we consider solely the first, second, and fourth terms of \eqref{eq:raw-lower-bound} in order to show a lower bound with respect to $\sdiff$. Again, we consider three cases:

\paragraph{Case 1 $\eta \geq \frac{2}{H}$:}
Again, in this case, from the second term of \eqref{eq:raw-lower-bound} we see that 
\begin{equation}
\E F(\hat{x}) - F^* \geq \frac{\mu c^2}{2} 
\end{equation}

\paragraph{Case 2 $\eta \leq \frac{2}{H}$ and $\prn*{1-\mu\eta}^{K} > \frac{H-2\mu}{H}$:}
In this case, from the first term of \eqref{eq:raw-lower-bound}, we have
\begin{align}
\E F(\hat{x}) - F^* 
&\geq \frac{\mu c^2\prn*{1-\mu\eta}^{KR}}{2} \\
&\geq \frac{\mu c^2}{2}\prn*{1-\frac{2\mu}{H}}^R \\
&\geq \frac{\mu c^2}{2}\prn*{\prn*{1-\frac{4\mu}{H}\prn*{1-\frac{1}{e}}}^{\frac{H}{4\mu}}}^{\frac{4\mu R}{H}} \\
&\geq \frac{\mu c^2}{2}\exp\prn*{-\frac{4\mu R}{H}}
\end{align}

\paragraph{Case 3 $\eta \leq \frac{2}{H}$ and $\prn*{1-\mu\eta}^{K} \leq \frac{H-2\mu}{H}$:}
In this case, from the first and fourth terms of \eqref{eq:raw-lower-bound}, we have
\begin{align}
\E F(\hat{x}) - F^* 
\geq \frac{\mu c^2}{2}\prn*{1-\mu\eta}^{KR} + \frac{\sdiff^2(H+2\mu)}{32\mu^2}\prn*{\frac{H-2\mu}{H} - \prn*{1-\mu\eta}^{K}}^2
\end{align}
Suppose that $\prn*{1-\mu\eta}^{K} \geq \frac{H-2\mu}{H} - \frac{1}{4R}$, then 
\begin{align}
\frac{\mu c^2}{2}\prn*{1-\mu\eta}^{KR}
\geq \frac{\mu c^2}{2}\prn*{1-\frac{2\mu}{H} - \frac{1}{4R}}^R
\end{align}
Then, if $R \geq \frac{H}{4\mu}$, then
\begin{align}
\frac{\mu c^2}{2}\prn*{1-\mu\eta}^{KR}
\geq \frac{\mu c^2}{2}\prn*{1-\frac{3\mu}{H}}^R 
\geq \frac{\mu c^2}{2}\prn*{\prn*{1-\frac{6\mu}{H}\prn*{1-\frac{1}{e}}}^{\frac{H}{6\mu}}}^{\frac{6\mu R}{H}} 
\geq \frac{\mu c^2}{2}\exp\prn*{-\frac{6\mu R}{H}}
\end{align}
Otherwise, if $R \leq \frac{H}{4\mu}$, then
\begin{align}
\frac{\mu c^2}{2}\prn*{1-\mu\eta}^{KR}
\geq \frac{\mu c^2}{2}\prn*{1-\frac{1}{2R}}^R \geq \frac{\mu c^2}{4} \geq \frac{\mu c^2}{4}\exp\prn*{-\frac{6\mu R}{H}}
\end{align}
Therefore, when $\prn*{1-\mu\eta}^{K} \geq \frac{H-2\mu}{H} - \frac{1}{4R}$,
\begin{equation}
\E F(\hat{x}) - F^* \geq \frac{\mu c^2}{4}\exp\prn*{-\frac{6\mu R}{H}}
\end{equation}

On the other hand, if $\prn*{1-\mu\eta}^{K} \leq \frac{H-2\mu}{H} - \frac{1}{4R}$, then
\begin{align}
\E F(\hat{x}) - F^* 
&\geq \frac{\sdiff^2(H+2\mu)}{32\mu^2}\prn*{\frac{H-2\mu}{H} - \prn*{1-\mu\eta}^{K}}^2 \\
&\geq \frac{\sdiff^2(H+2\mu)}{32\mu^2}\prn*{\frac{1}{4R}}^2 \\
&\geq\frac{H\sdiff^2}{512\mu^2R^2}
\end{align}

\paragraph{Combination:}
Combining these three cases, we conclude that
\begin{align}
\E F(\hat{x}) - F^* 
&\geq \min\crl*{\frac{\mu c^2}{4}\exp\prn*{-\frac{6\mu R}{H}},\, \frac{H\sdiff^2}{512\mu^2R^2}}
\end{align}

In the strongly convex case, we recall that $F(0) - F(x^*) = \mu c^2$, so we choose $\mu = \lambda$ and $c^2 = \frac{\Delta}{\lambda}$ so that the objective satisfies the strongly convex assumptions. Now, the lower bound reads (for a universal constant $\beta$)
\begin{align}
\E F(\hat{x}) - F^* 
&\geq \beta\cdot\min\crl*{\Delta\exp\prn*{-\frac{6\lambda R}{H}},\, \frac{H\sdiff^2}{512\lambda^2R^2}}
\end{align}

In the convex case, we recall that $\nrm{x^*}^2 \leq 2c^2$, so we choose $c^2 = \frac{B}{2}$ so that the convex assumptions are satisfied. We now have two options, if $R \leq \frac{H^2B^2}{\sdiff^2}$, then we pick $\mu = \prn*{\frac{H\sdiff^2}{B^2R^2}}^{1/3}$ so that the lower bound reads
\begin{align}
\E F(\hat{x}) - F^* 
&\geq \beta\cdot \frac{\prn*{H\sdiff^2B^4}^{1/3}}{R^{2/3}}\exp\prn*{-\frac{6\sdiff^{2/3}R^{1/3}}{H^{2/3}B^{2/3}}} \\
&\geq \beta\cdot \frac{\prn*{H\sdiff^2B^4}^{1/3}}{R^{2/3}}\exp\prn*{-6} \\
&\geq \beta'\cdot \frac{\prn*{H\sdiff^2B^4}^{1/3}}{R^{2/3}}
\end{align}
On the other hand, if $R \geq \frac{H^2B^2}{\sdiff^2}$, then we pick $\mu = \frac{H}{6R}$ so the lower bound reads
\begin{align}
\E F(\hat{x}) - F^* 
\geq \beta\cdot\min\crl*{\frac{HB^2}{R},\, \frac{\sdiff^2}{H}} 
= \beta\cdot\frac{HB^2}{R}
\end{align}
Consequently, 
\begin{align}
\E F(\hat{x}) - F^* 
\geq \beta\cdot\min\crl*{\frac{HB^2}{R},\, \frac{\prn*{H\sdiff^2 B^4}^{1/3}}{R^{2/3}}}
\end{align}
Combining these with the stochastic terms completes the proof.
\end{proof}

\section{Proof of Theorem \ref{thm:local-sgd-uppper-bound}}\label{app:zeta-everywhere-upper-bound}
We prove the theorem with the help of several technical lemmas.
\begin{lemma}\label{lem:zeta-everywhere-progress}
For any stepsize $\eta_t \leq \frac{1}{10H}$
\[
\E\brk*{F(\bx_t) - F^*} \leq \prn*{\frac{1}{\eta_t} - \lambda}\E\nrm*{\bx_t - x^*}^2 - \frac{1}{\eta_t}\E\nrm*{\bx_{t+1} - x^*}^2 + \frac{3\sigma_*^2\eta_t}{M} + \frac{2H}{M}\sum_{m=1}^M\E\nrm*{\bx_t - x_t^m}^2
\]
\end{lemma}
\begin{proof}
This lemma and its proof are nearly identical to \cite[Lemma 8][]{koloskova2020unified}. We include a proof here in order to keep the paper self-contained.

Let $\bx_{t+1} = \frac{1}{M}\sum_{m=1}^Mx_t^m$ be the average of the machines' local iterates at time $t$. Then,
\begin{align}
\E\nrm*{\bx_{t+1} - x^*}^2
&= \E\nrm*{\bx_t - \frac{\eta_t}{M}\sum_{m=1}^M\nabla F_m(x_t^m) - x^*}^2 + \eta_t^2\E\nrm*{\frac{1}{M}\sum_{m=1}^M\nabla f(x_t^m;z_t^m) - \nabla F_m(x_t^m)}^2 \label{eq:koloskova8-1}
\end{align}
Beginning with the first term of \eqref{eq:koloskova8-1}:
\begin{align}
\E&\nrm*{\bx_t - \frac{\eta_t}{M}\sum_{m=1}^M\nabla F_m(x_t^m) - x^*}^2\nonumber\\
&= \E\nrm*{\bx_t - x^*}^2 + \eta_t^2\E\nrm*{\frac{1}{M}\sum_{m=1}^M\nabla F_m(x_t^m)}^2 - \frac{2\eta_t}{M}\sum_{m=1}^M\E\inner{\bx_t - x^*}{\nabla F_m(x_t^m)} \label{eq:koloskova8-2}
\end{align}
We can bound the second term of \eqref{eq:koloskova8-2} with:
\begin{align}
\eta_t^2&\E\nrm*{\frac{1}{M}\sum_{m=1}^M\nabla F_m(x_t^m)}^2 \nonumber\\
&\leq 2\eta_t^2\E\nrm*{\frac{1}{M}\sum_{m=1}^M\nabla F_m(x_t^m) - \nabla F_m(\bx_t)}^2 + 2\eta_t^2\E\nrm*{\frac{1}{M}\sum_{m=1}^M\nabla F_m(\bx_t) - \nabla F_m(x^*)}^2 \\
&\leq \frac{2\eta_t^2}{M}\sum_{m=1}^M\E\nrm*{\nabla F_m(x_t^m) - \nabla F_m(\bx_t)}^2 + 2\eta_t^2\E\nrm*{\nabla F(\bx_t) - \nabla F(x^*)}^2 \\
&\leq \frac{2H^2\eta_t^2}{M}\sum_{m=1}^M\E\nrm*{x_t^m - \bx_t}^2 + 4H\eta_t^2\E\brk*{F(\bx_t) - F(x^*)}
\end{align}
For the third term of \eqref{eq:koloskova8-2}:
\begin{align}
-&\frac{2\eta_t}{M}\sum_{m=1}^M\E\inner{\bx_t - x^*}{\nabla F_m(x_t^m)} \nonumber\\
&= -\frac{2\eta_t}{M}\sum_{m=1}^M\E\inner{x_t^m - x^*}{\nabla F_m(x_t^m)} + \frac{2\eta_t}{M}\sum_{m=1}^M\E\inner{x_t^m - \bx_t}{\nabla F_m(x_t^m)} \\
&\leq -\frac{2\eta_t}{M}\sum_{m=1}^M\E\brk*{F_m(x_t^m) - F_m(x^*) + \frac{\lambda}{2}\nrm*{x_t^m - x^*}^2} \nonumber\\
&\qquad+ \frac{2\eta_t}{M}\sum_{m=1}^M\E\brk*{F_m(x_t^m) - F_m(\bx_t) + \frac{H}{2}\nrm*{x_t^m - \bx_t}^2} \\
&\leq -2\eta_t\E\brk*{F(\bx_t) - F(x^*) + \frac{\lambda}{2}\nrm*{\bx_t - x^*}^2} +  \frac{H\eta_t}{M}\sum_{m=1}^M\nrm*{x_t^m - \bx_t}^2
\end{align}
Finally, for the second term of \eqref{eq:koloskova8-1}
\begin{align}
\eta_t^2&\E\nrm*{\frac{1}{M}\sum_{m=1}^M\nabla f(x_t^m;z_t^m) - \nabla F_m(x_t^m)}^2\nonumber\\
&= \frac{\eta_t^2}{M^2}\sum_{m=1}^M\E\nrm*{\nabla f(x_t^m;z_t^m) - \nabla F_m(x_t^m)}^2 \\
&\leq \frac{3\eta_t^2}{M^2}\sum_{m=1}^M\bigg[\E\nrm*{\nabla f(x_t^m;z_t^m) - \nabla f(\bx_t;z_t^m)}^2 + \E\nrm*{\nabla f(\bx_t;z_t^m) - \nabla f(x^*;z_t^m)}^2 \nonumber\\
&\qquad\qquad+ \E\nrm*{\nabla f(x^*;z_t^m) - \nabla F_m(x^*)}^2\bigg] \\
&\leq \frac{3\eta_t^2}{M^2}\sum_{m=1}^M\brk*{H^2\E\nrm*{x_t^m - \bx_t}^2 + 2H\E\brk*{F_m(\bx_t) - F_m(x^*)} + \sigma_*^2} \\
&\leq \frac{3\eta_t^2}{M}\sum_{m=1}^M\brk*{H^2\E\nrm*{x_t^m - \bx_t}^2 + 2H\E\brk*{F(\bx_t) - F(x^*)} + \frac{\sigma_*^2}{M}}
\end{align}
Combining all these results back into \eqref{eq:koloskova8-1}, we have
\begin{align}
\E\nrm*{\bx_{t+1} - x^*}^2
&\leq \prn*{1-\lambda\eta_t}\E\nrm*{\bx_t - x^*}^2 + \frac{H\eta_t + 5H^2\eta_t^2}{M} \sum_{m=1}^M \E\nrm*{x_t^m - \bx_t}^2 \nonumber\\
&\quad+ (10H\eta_t^2 - 2\eta_t)\E\brk*{F(\bx_t) - F(x^*)} + \frac{3\eta_t^2\sigma_*^2}{M} \\
&\leq \prn*{1-\lambda\eta_t}\E\nrm*{\bx_t - x^*}^2 + \frac{2H\eta_t}{M} \sum_{m=1}^M \E\nrm*{x_t^m - \bx_t}^2 \nonumber\\
&\quad - \eta_t\E\brk*{F(\bx_t) - F(x^*)} + \frac{3\eta_t^2\sigma_*^2}{M}
\end{align}
where for the final line we used that $\eta_t \leq \frac{1}{10H}$. Rearranging completes the proof.
\end{proof}

\begin{lemma}\label{lem:zeta-everywhere-divergence}
If $\sup_{x,m}\nrm*{\nabla F_m(x) - \nabla F(x)}^2 \leq \bar{\zeta}^2$, then for any fixed stepsize $\eta$
\[
\frac{1}{M}\sum_{m=1}^M \E\nrm*{x_t^m - \bx_t}^2 \leq 3K\sigma^2\eta^2 + 6K^2\eta^2\bar{\zeta}^2
\]
Similarly, the decreasing stepsize $\eta_t = \frac{2}{\lambda(a + t + 1)}$ for any $a$
\[
\frac{1}{M}\sum_{m=1}^M \E\nrm*{x_t^m - \bx_t}^2 \leq 3K\sigma^2\eta_{t-1}^2 + 6K^2\bar{\zeta}^2\eta_{t-1}^2
\]
\end{lemma}
\begin{proof}
By Jensen's inequality
\begin{equation}
\E\nrm*{x_t^m - \bx_t}^2 \leq \frac{1}{M}\sum_{n=1}^M\E\nrm*{x_t^m - x_t^n}^2
\end{equation}
Therefore, it suffices to bound $\E\nrm*{x_t^m - x_t^n}^2$, which we do now:
\begin{align}
&\E\nrm*{x_t^m - x_t^n}^2 \nonumber\\
&\leq \E\left\|x_{t-1}^m - x_{t-1}^n - \eta_{t-1}\prn*{\nabla F(x_{t-1}^m) - \nabla F(x_{t-1}^n)} \right.\nonumber\\
&+ \left.\eta_{t-1}\prn*{\nabla F(x_{t-1}^m) - \nabla F_m(x_{t-1}^m) - \nabla F(x_{t-1}^n) + \nabla F_n(x_{t-1}^n)}\right\|^2 + \eta_{t-1}^2\sigma^2 \\
&\leq \inf_{\gamma > 0} \prn*{1 + \frac{1}{\gamma}}\E\nrm*{x_{t-1}^m - x_{t-1}^n - \eta_{t-1}\prn*{\nabla F(x_{t-1}^m) - \nabla F(x_{t-1}^n)}}^2 \nonumber\\
&\qquad+ \prn*{1+\gamma}\eta_{t-1}^2\E\nrm*{\nabla F(x_{t-1}^m) - \nabla F_m(x_{t-1}^m) - \nabla F(x_{t-1}^n) + \nabla F_n(x_{t-1}^n)}^2 + \eta_{t-1}^2\sigma^2 \\
&\leq  \inf_{\gamma > 0} \prn*{1 + \frac{1}{\gamma}}\prn*{1 - \lambda\eta_{t-1}}\E\nrm*{x_{t-1}^m - x_{t-1}^n}^2 + \eta_{t-1}^2\sigma^2 \nonumber\\
&+ \prn*{1+\gamma}\eta_{t-1}^2\E\nrm*{\nabla F(x_{t-1}^m) - \nabla F_m(x_{t-1}^m)}^2 \nonumber\\
&+ \prn*{1+\gamma}\eta_{t-1}^2\E\nrm*{\nabla F(x_{t-1}^n) - \nabla F_n(x_{t-1}^n)}^2 \nonumber\\
&- 2\prn*{1+\gamma}\eta_{t-1}^2\E\inner{\nabla F(x_{t-1}^m) - \nabla F_m(x_{t-1}^m)}{\nabla F(x_{t-1}^n) - \nabla F_n(x_{t-1}^n)} 
\end{align}
For the third inequality we used Lemma \ref{lem:co-coercivity}. Therefore,
\begin{align}
\frac{1}{M^2}&\sum_{m=1}^M\sum_{n=1}^M \E\nrm*{x_t^m - x_t^n}^2 \nonumber\\
&\leq \frac{1}{M^2}\sum_{m=1}^M\inf_{\gamma > 0} \prn*{1 + \frac{1}{\gamma}}\prn*{1 - \lambda\eta_{t-1}}\E\nrm*{x_{t-1}^m - x_{t-1}^n}^2 + \eta_{t-1}^2\sigma^2 + 2\prn*{1+\gamma}\eta_{t-1}^2\bar{\zeta}^2
\end{align}
We will unroll this recurrence, using that $x_{t_0}^m = x_{t_0}^n$ for all $m,n$ where $t_0$ is the most recent time that the iterates were synchronized, so $t - t_0 \leq K-1$. 
Taking $\gamma = K-1$, we have
\begin{align}
\frac{1}{M^2}\sum_{m=1}^M\sum_{n=1}^M \E\nrm*{x_t^m - x_t^n}^2
&= \sum_{i=t_0}^{t-1} \prn*{\eta_i^2\sigma^2 + 2(1+\gamma)\eta_i^2\bar{\zeta}^2}\prod_{j=i+1}^{t-1}\prn*{1+\frac{1}{\gamma}}\prn*{1-\lambda\eta_j}\\
&\leq \sum_{i=t_0}^{t-1} \prn*{\eta_i^2\sigma^2 + 2K\eta_i^2\bar{\zeta}^2}\prod_{j=i+1}^{t-1}\prn*{1+\frac{1}{K-1}}\prn*{1-\lambda\eta_j} \\
&\leq \sum_{i=t_0}^{t-1} \prn*{\eta_i^2\sigma^2 + 2K\eta_i^2\bar{\zeta}^2}\prn*{1+\frac{1}{K-1}}^{K-1} \prod_{j=i+1}^{t-1}\prn*{1-\lambda\eta_j} \\
&\leq 3\prn*{\sigma^2 + 2K\bar{\zeta}^2}\sum_{i=t_0}^{t-1} \eta_i^2\prod_{j=i+1}^{t-1}\prn*{1-\lambda\eta_j}
\end{align}
For a constant stepsize $\eta$, 
\begin{align}
\frac{1}{M^2}\sum_{m=1}^M\sum_{n=1}^M \E\nrm*{x_t^m - x_t^n}^2
&\leq 3\prn*{\sigma^2 + 2K\bar{\zeta}^2}\sum_{i=t_0}^{t-1} \eta^2 \\
&\leq 3K\prn*{\sigma^2 + 2K\bar{\zeta}^2}\eta^2
\end{align}
For decreasing stepsize $\eta_t = \frac{2}{\lambda(a+t+1)}$
\begin{align}
\frac{1}{M^2}\sum_{m=1}^M\sum_{n=1}^M \E\nrm*{x_t^m - x_t^n}^2
&\leq 3\prn*{\sigma^2 + 2K\bar{\zeta}^2}\sum_{i=t_0}^{t-1} \eta_i^2\prod_{j=i+1}^{t-1}\frac{a+j-1}{a+j+1} \\
&= 3\prn*{\sigma^2 + 2K\bar{\zeta}^2}\sum_{i=t_0}^{t-1} \eta_i^2\frac{(a+i)(a+i+1)}{(a+t)(a+t+1)} \\
&= 3\prn*{\sigma^2 + 2K\bar{\zeta}^2}\sum_{i=t_0}^{t-1} \eta_i^2\frac{\eta_{t-1}\eta_t}{\eta_{i-1}\eta_i} \\
&\leq 3\prn*{\sigma^2 + 2K\bar{\zeta}^2}\sum_{i=t_0}^{t-1} \eta_i^2\frac{\eta_{t-1}^2}{\eta_i^2} \\
&= 3K\prn*{\sigma^2 + 2K\bar{\zeta}^2}\eta_{t-1}^2
\end{align}

\end{proof}

\localsgdupperbound*
\begin{proof}
By Lemma \ref{lem:zeta-everywhere-progress}, for any $\eta_t \leq \frac{1}{10H}$
\begin{equation}
\E\brk*{F(\bx_t) - F^*} \leq \prn*{\frac{1}{\eta_t} - \lambda}\E\nrm*{\bx_t - x^*}^2 - \frac{1}{\eta_t}\E\nrm*{\bx_{t+1} - x^*}^2 + \frac{3\sigma_*^2\eta_t}{M} + \frac{2H}{M}\sum_{m=1}^M\E\nrm*{\bx_t - x_t^m}^2
\end{equation}
By Lemma \ref{lem:zeta-everywhere-divergence}, when $\eta_t=\eta$ is constant then
\begin{equation}
\frac{1}{M}\sum_{m=1}^M \E\nrm*{x_t^m - \bx_t}^2 \leq 3K\sigma^2\eta^2 + 6K^2\eta^2\bar{\zeta}^2
\end{equation}
and when $\eta_t = \frac{2}{\lambda(a + t + 1)}$
\begin{equation}
\frac{1}{M}\sum_{m=1}^M \E\nrm*{x_t^m - \bx_t}^2 \leq 3K\sigma^2\eta_{t-1}^2 + 6K^2\bar{\zeta}^2\eta_{t-1}^2
\end{equation}
We now consider the convex and strongly convex cases separately:

\paragraph{Convex case:}
In the convex case, we use a constant stepsize $\eta$, so 
\begin{align}
\E&\brk*{F(\bx_t) - F^*} \nonumber\\
&\leq \frac{1}{\eta}\E\nrm*{\bx_t - x^*}^2 - \frac{1}{\eta}\E\nrm*{\bx_{t+1} - x^*}^2 + \frac{3\sigma_*^2\eta}{M} + \frac{2H}{M}\sum_{m=1}^M\E\nrm*{\bx_t - x_t^m}^2 \\
&\leq \frac{1}{\eta}\E\nrm*{\bx_t - x^*}^2 - \frac{1}{\eta}\E\nrm*{\bx_{t+1} - x^*}^2 + \frac{3\sigma_*^2\eta}{M} + 6HK\sigma^2\eta^2 + 12HK^2\eta^2\bar{\zeta}^2
\end{align}
Therefore, by the convexity of $F$
\begin{align}
\E\brk*{F\prn*{\frac{1}{KR}\sum_{t=1}^{KR} \bx_t} - F^*} 
&\leq \frac{1}{KR}\sum_{t=1}^{KR}\E\brk*{F(\bx_t) - F^*} \\
&\leq \frac{B^2}{\eta KR} + \frac{3\sigma_*^2\eta}{M} + 6HK\sigma^2\eta^2 + 12HK^2\eta^2\bar{\zeta}^2
\end{align}
Choosing
\begin{equation}
\eta = \min\crl*{\frac{1}{10H},\, \frac{B\sqrt{M}}{\sigma_*\sqrt{KR}},\, \prn*{\frac{B^2}{HK^2\sigma^2}}^{1/3},\, \prn*{\frac{B^2}{HK^2\bar{\zeta}^2}}^{1/3}}
\end{equation}
then ensures
\begin{align}
\E\brk*{F\prn*{\frac{1}{KR}\sum_{t=1}^{KR} \bx_t} - F^*} 
&\leq \frac{10HB^2}{KR} \frac{13\prn*{H\bar{\zeta}^2B^4}^{1/3}}{R^{2/3}} + \frac{7\prn*{H\sigma^2B^4}^{1/3}}{K^{1/3}R^{2/3}} + \frac{4\sigma_*B}{\sqrt{MKR}}
\end{align}

\paragraph{Strongly convex case:}
In the strongly convex case, we take the stepsize $\eta_t = \frac{2}{\lambda(a+t+1)}$ for $a = 20H/\lambda$ which ensures $\eta_t \leq \frac{1}{10H}$. In addition, we define weights $w_t = (a+t)$ and define
\begin{equation}
\bar{x} = \frac{1}{W}\sum_{t=1}^{KR}w_t\bx_t
\end{equation}
where $W = \sum_{t=1}^{KR}w_t \geq \frac{1}{2}KR(a+KR)$.
By the convexity of $F$,
\begin{align}
\E& F(\bar{x}) - F^* \nonumber\\
&\leq \frac{1}{W}\sum_{t=1}^{KR} (a+t)\E F(\bx_t) - F^* \\
&\leq \frac{\lambda (a+1)(a+2)B^2}{2W} + \frac{1}{W}\sum_{t=1}^{KR}\brk*{\frac{6\sigma_*^2}{\lambda M} + \frac{2H(a+t)}{M}\sum_{m=1}^M\E\nrm*{\bx_t - x_t^m}^2} \\
&\leq \frac{\lambda (a+1)(a+2)B^2}{2W} + \frac{6\sigma_*^2KR}{W\lambda M} + \frac{6HK\sigma^2 + 12HK^2\bar{\zeta}^2}{W}\sum_{t=1}^{KR}(a+t)\eta_{t-1}^2 \\
&\leq \frac{\lambda (a+1)(a+2)B^2}{2W} + \frac{6\sigma_*^2KR}{W\lambda M} + \frac{6HK\sigma^2 + 12HK^2\bar{\zeta}^2}{\lambda^2 W}\prn*{1 + \log\prn*{a+KR}} \\
&\leq \frac{132H^2 B^2}{\lambda KR(10H/\lambda + KR)} + \prn*{\frac{12H\bar{\zeta}^2}{\lambda^2 R^2} + \frac{6H\sigma^2}{\lambda^2 KR^2}} \log\prn*{\frac{13H}{\lambda}+KR} + \frac{6\sigma_*^2}{\lambda MKR}
\end{align}
Note that in the strongly convex case, it is likely possible to achieve a first term scaling with $\exp(-KR)$ using a method similar to Lemma \ref{lem:sc-exp-rate}. However, the recurrence we derived here has a different form, and it is difficult to determine the correct stepsize and weighting schedule to achieve linear convergence.
\end{proof}

\section{Details of Experiments}\label{app:experiments}
The training set of MNIST (60,000 examples) was divided by digit into ten groups of equal size $n \approx 6,000$ (which required discarding some examples from the more common digits). PCA was used to reduce the dimensionality to 100, but no other preprocessing was used.

Then, for each of the 25 combinations ($i$,$j$) for even $i$ and odd $j$, a binary classification ``task'' was created, i.e.~classifying even ($+1$) versus odd ($-1$). These tasks were arbitrarily labelled task $1,2,\dots,25$.

For each $p \in [0.0, 0.2, 0.4, 0.6, 0.8, 1.0]$, machine $m$ was assigned data composed of $p\cdot 2n$ random examples from task $m$, and $(1-p)\cdot 2n$ random examples from a mixture of all the tasks. 

Local and Minibatch SGD were then used to optimize the logistic loss for each of the six described local datasets. The constant stepsize was tuned (from a log-scale grid of 10 points ranging from $e^{-6},\dots,e^{0}$ for Minibatch SGD, and a log-scale grid of 10 points ranging from $e^{-8},\dots,e^{-1}$ for Local SGD) for each value of $p$, $K$, and $R$ individually, and the average loss over four runs is reported for the best stepsize for each point in the plot. That is, each point in the plot represents the best possible performance of the algorithm for that $p$, $K$, and $R$ specifically.

Finally, we computed the value of $\zeta_*^2$ as a function of $p$ by using Newton's method to compute a very accurate estimate of the minimizer, and then explicitly calculating $\zeta_*^2(p)$ at that point.


\section{Proof of Theorem \ref{thm:dzr-lower-bound}}\label{app:alg-independent-lower-bound}
For this lower bound, the gradients will always be noiseless, so we simply define the expectation of the local functions. Furthermore, we will construct just two local functions $F_1$ and $F_2$. For the case $M > 2$, $F_1$ will be assigned to the first $\lfloor M/2\rfloor$ machines, and $F_2$ to the next $\lfloor M/2\rfloor$ machines. If there is an odd number of machines, we simply assign the last machine $F_3(x) = \frac{\lambda}{2}\nrm{x}^2$, which will reduce the lower bound by a factor of at most $\frac{M-1}{M}$. Therefore, we proceed by focusing on the case $M=2$. 

We define the following $H$-smooth and $\lambda$-strongly convex functions on $\R^d$ for even $d$:
\begin{align}
F(x) &= \frac{1}{2}\prn*{F_1(x) + F_2(x)} \label{eq:dzr-construction}\\
F_1(x) &= \frac{H-\lambda}{8}\prn*{x_1^2 - 2Cx_1 + \beta x_d^2 + \sum_{i=1}^{d/2-1}\prn*{x_{2i+1} - x_{2i}}^2} + \frac{\lambda}{2}\nrm*{x}^2 \\
F_2(x) &= \frac{H-\lambda}{8}\prn*{\sum_{i=1}^{d/2}\prn*{x_{2i} - x_{2i-1}}^2} + \frac{\lambda}{2}\nrm*{x}^2
\end{align}
Here, $\beta$ and $C$ are constants which will be chosen later.

These functions are identical to ones used by \citet{woodworth16tight} to prove lower bounds for finite sum optimization, and are very similar both to classic work by \citet{nesterov2004introductory} on lower bounds and to more closely related work by \citet{arjevani2015communication}. \citeauthor{arjevani2015communication} also prove lower bounds for distributed optimization algorithms, but their slightly different construction made it more difficult to tune $\sdiff^2$, which is necessary for our lower bound.

These functions have the following important property: let $E_k = \spn\crl{e_1,\dots,e_k}$ be the set of vectors whose $k+1,\dots,d$ coordinates are all zero, then for all $x_k \in E_k$ for even $k$
\begin{equation}
\nabla F_1(x_k) \in E_{k+1}\qquad\textrm{and}\qquad \nabla F_2(x_k) \in E_k\label{eq:grad-span-even}
\end{equation}
and for $x_k \in E_k$ for odd $k$
\begin{equation}
\nabla F_1(x_k) \in E_{k}\qquad\textrm{and}\qquad \nabla F_2(x_k) \in E_{k+1}\label{eq:grad-span-odd}
\end{equation}
For algorithms whose iterates, for example, remain in the span of previous gradients, the only way to access the next coordinate is to query the gradient of one of the two functions---$F_1$ if the next coordinate is odd, and $F_2$ if the next coordinate is even. Since each machine will only have access to one of the two functions throughout each round of communication, this means that each round of communication can only unlock a single new coordinate. We now formalize this.

Following \citet{carmon2017lower1}, we define:
\begin{definition}[Distributed zero-respecting algorithm]
For a vector $v$, let $\supp(v) = \crl*{i \in \crl{1,\dots,d}\,:\, v_i \neq 0}$. We say that an optimization algorithm is distributed zero-respecting if for all $t$ and $m$, the $t$th query on the $m$th machine, $x_t^m$ satisfies
\[
\supp(x_t^m) \subseteq \bigcup_{s < t} \supp(\nabla f(x_s^m;z_s^m)) \cup \bigcup_{m' \neq m}\bigcup_{s \leq \pi_m(t,m')}\supp(\nabla f(x_s^{m'};z_s^{m'}))
\]
where $\pi_m(t,m')$ is the most recent time before $t$ when machines $m$ and $m'$ communicated with each other.
\end{definition}
This definition captures a very wide variety of distributed optimization algorithms, including minibatch SGD, accelerated minibatch SGD, local SGD, coordinate descent methods, and many more. Algorithms which are \emph{not} distributed zero-respecting are those whose iterates have components in directions about which the algorithm has no information, meaning that in some sense, it is just ``wild guessing.'' Using techniques similar to \citet[Theorem 7][]{woodworth16tight} and \citet{carmon2017lower1}, it should be possible to extend this lower bound beyond distributed zero-respecting algorithms to arbitrary randomized algorithms. 

We now argue that the progress of distributed zero-respecting algorithms is controlled by the number of rounds of communication, $R$, regardless of $K$:
\begin{lemma}\label{lem:dzr-progress}
Let $\hat{x}$ be the output after $R$ rounds of communication of a distributed zero-respecting algorithm optimizing $F = \frac{1}{2}(F_1+F_2)$ as defined in \eqref{eq:dzr-construction}. Then,
\[
\supp(x_t^m) \in E_R
\]
\end{lemma}
\begin{proof}
The definition of a zero-respecting algorithm requires that every machine's initial iterate $x_0^m = 0$. We will now prove the Lemma by induction on the round of communication.

As a base case, for the first iteration of the first round of communication:
\begin{equation}
\nabla F_1(x_0^1) = \nabla F_1(0) = \frac{(\lambda - H)C}{4}e_1 \in E_1\qquad\textrm{and}\qquad \nabla F_2(x_1^2) = \nabla F_2(0) = 0 \in E_0
\end{equation}
Therefore, by the distributed zero-respecting property, $x_2^1 \in E_1$ and $x_2^2 \in E_0$. Furthermore, for all $y_1 \in E_1$, $\nabla F_1(y_1) \in E_1$ and for all $y_0 \in E_0$, $\nabla F_2(y_0) \in E_0$. Therefore, further gradient queries on each machine will not change the set of coordinates that the distributed zero-respecting property allows to be non-zero. We conclude that $x_t^1 \in E_1$ and $x_t^2 \in E_0$ for all $t$ until machines $1$ and $2$ communicate with each other.

Now, suppose that after $r-1$ rounds of communication, $x_t^1, x_t^2 \in E_{r-1}$. If $r$ is even, then $\nabla F_1(x_t^1) \in E_{r-1}$ and $\nabla F_2(x_t^2) \in E_r$. Furthermore, additional gradient computations within the $r$th round of communication will not expand the set of coordinates that the distributed zero-respecting property will allow to be non-zero. Therefore, both machines' coordinates will remain in $E_r$ until the end of the $r$th round of communication. A similar argument can be made for odd $r$.
\end{proof}

Now, we will compute the minimizer of $F$. We note that by the definition of $F_1$ and $F_2$, 
\begin{equation}
F(x) = \frac{H-\lambda}{16}\prn*{x_1^2 - 2Cx_1 + \beta x_d^2 + \sum_{i=2}^{d }\prn*{x_{i} - x_{i-1}}^2} + \frac{\lambda}{2}\nrm*{x}^2
\end{equation}
Calculating the gradient of $F$, we see that $x^* = \argmin_x F(x)$ must satisfy
\begin{equation}
\begin{aligned}
C &= \prn*{2 + \frac{8\lambda}{H-\lambda}} x_1^* - x_2^* \\
0 &= \prn*{2 + \frac{8\lambda}{H-\lambda}} x_i^* - x_{i+1}^* - x_{i-1}^* \qquad\forall_{i\in\crl{2,\dots,d-1}} \\
0 &= \prn*{1 + \beta + \frac{8\lambda}{H-\lambda}} x_d^* - x_{d-1}^*\label{eq:dzr-solution-conditions}
\end{aligned}
\end{equation}
Let $q$ be the smaller solution of the quadratic equation
\begin{equation}\label{eq:q-quadratic-equation}
1 - \prn*{2 + \frac{8\lambda}{H-\lambda}}q + q^2 = 0
\end{equation}
That is,
\begin{align}
q 
&= 1 + \frac{4\lambda}{H-\lambda} - \sqrt{\frac{16\lambda^2}{(H-\lambda)^2} + \frac{8\lambda}{H-\lambda}} \\
&= 1 + \frac{4\lambda}{H-\lambda}\prn*{1 - \sqrt{1 + \frac{H-\lambda}{2\lambda}}} \\
&= 1 - \frac{2}{\prn*{1 - \sqrt{1 + \frac{H-\lambda}{2\lambda}}}\prn*{1 + \sqrt{1 + \frac{H-\lambda}{2\lambda}}}}\prn*{1 - \sqrt{1 + \frac{H-\lambda}{2\lambda}}} \\
&= \frac{\sqrt{1 + \frac{H-\lambda}{2\lambda}} - 1}{\sqrt{1 + \frac{H-\lambda}{2\lambda}} + 1}
\end{align}
Let $\alpha = \sqrt{1 + \frac{H-\lambda}{2\lambda}}$ so that $q = \frac{\alpha - 1}{\alpha + 1}$, and define $\beta = 1 - q$. Then it is straightforward to confirm that 
\begin{equation}
x^* = C\sum_{i=1}^d q^i e_i
\end{equation}
satisfies all of the conditions \eqref{eq:dzr-solution-conditions}, and is thus the minimizer of $F$. This point has value
\begin{align}
F(x^*)
&= \frac{C^2(H-\lambda)}{16}\prn*{q^2 - 2q + \beta q^{2d} + (1-q)^2\sum_{i=2}^d q^{2i-2} + \frac{8\lambda}{H-\lambda}\sum_{i=1}^d q^{2i}} \\
&= \frac{C^2(H-\lambda)}{16}\prn*{-1 + \beta q^{2d} + (1-q)^2\sum_{i=1}^d q^{2i-2} + \frac{8\lambda}{H-\lambda}\sum_{i=1}^d q^{2i}} \\
&= \frac{C^2(H-\lambda)}{16}\prn*{-1 + (1-q)q^{2d} + \frac{8\lambda}{H-\lambda}\sum_{i=1}^d q^{2i-1} + q^{2i}} \\
&= \frac{C^2(H-\lambda)}{16}\prn*{-1 + (1-q)q^{2d} + \frac{8\lambda}{H-\lambda}\prn*{\frac{q(1 - q^{2d})}{1-q^2} + \frac{q^2(1-q^{2d})}{1-q^2}}} \\
&= \frac{C^2(H-\lambda)}{16}\prn*{-1 + (1-q)q^{2d} + \frac{(1-q)^2}{q}\prn*{\frac{q(1 - q^{2d})}{1-q^2} + \frac{q^2(1-q^{2d})}{1-q^2}}} \\
&= \frac{C^2(H-\lambda)}{16}\prn*{-1 + (1-q)q^{2d} + (1-q)(1 - q^{2d})} \\
&= \frac{-qC^2(H-\lambda)}{16}\label{eq:C-for-function-value}
\end{align}
For the third equality, we used that \eqref{eq:q-quadratic-equation} implies $(1-q)^2 = \frac{8\lambda q}{H-\lambda}$. For the fifth inequality, we used that $\frac{8\lambda}{H-\lambda} = \frac{(1-q)^2}{q}$.
This solution has norm
\begin{align}
\nrm*{x^*}^2 
= C^2\sum_{i=1}^{d}q^{2i} 
= C^2\frac{q^2(1-q^{2d})}{1-q^2} 
\leq \frac{q^2C^2}{1-q^2} 
= \frac{C^2(\alpha - 1)^2}{4\alpha} 
\leq \frac{\alpha C^2}{4}\label{eq:C-for-norm}
\end{align}
Furthermore,
\begin{equation}
F(0) - F(x^*) = -F(x^*) = \frac{qC^2(H-\lambda)}{16}\label{eq:C-for-suboptimality}
\end{equation}
Finally, we evaluate the degree of heterogeneity:
\begin{align}
\sdiff^2 &= \frac{1}{2}\sum_{m=1}^2 \nrm*{\nabla F_m(x^*)}^2 = \nrm*{\nabla F_1(x^*)}^2 = \nrm*{\nabla F_2(x^*)}^2 \\
&= \frac{(H-\lambda)^2}{64}\nrm*{2\sum_{i=1}^{d/2}\prn*{x_{2i}^* - x_{2i-1}^*}\prn*{e_{2i} - e_{2i-1}} + \frac{4\lambda}{H-\lambda}x^*}^2 \\
&= \frac{(H-\lambda)^2}{16}\sum_{i=1}^{d/2}\brk*{\prn*{x_{2i-1}^*\prn*{-1 + \frac{4\lambda}{H-\lambda}}}^2 + \prn*{x_{2i}^*\prn*{1 + \frac{4\lambda}{H-\lambda}}}^2} \\
&= \frac{C^2(H-\lambda)^2}{16}\sum_{i=1}^{d/2}\brk*{q^{4i-2}\frac{(H-5\lambda)^2}{(H-\lambda)^2} + q^{4i}\frac{(H+3\lambda)^2}{(H-\lambda)^2}} \\
&\leq \frac{(H+3\lambda)^2}{16}\nrm{x^*}^2 \\
&\leq \frac{\alpha C^2 (H+3\lambda)^2}{64}\label{eq:C-for-sdiff}
\end{align}
With this, we are ready to prove the lower bound.

\dzrlowerbound*
\begin{proof}
By Lemma \ref{lem:dzr-progress}, the output of the algorithm $\hat{x} \in E_R$. Furthermore, since $F$ is $\lambda$-strongly convex,
$F(\hat{x}) - F^* \geq \frac{\lambda}{2}\nrm*{\hat{x} - x^*}^2$. Therefore,
\begin{align}
\frac{F(\hat{x}) - F^*}{F(0) - F^*} 
&\geq \frac{\frac{\lambda}{2}\nrm*{\hat{x} - x^*}^2}{\frac{qC^2(H-\lambda)}{16}} \\
&\geq \frac{8\lambda}{q(H-\lambda)}\sum_{i=R+1}^d q^{2i} \\
&= \frac{8\lambda q (q^{2R} - q^{2d})}{(H-\lambda)(1-q^2)} \\
&= \frac{(1-q)^2(q^{2R} - q^{2d})}{1-q^2} \\
&= \frac{(1-q)(q^{2R} - q^{2d})}{1+q} \\
&= \frac{q^{2R} - q^{2d}}{\alpha}
\end{align}
For the third equality we used that \eqref{eq:q-quadratic-equation} implies $\frac{8\lambda q}{H-\lambda} = (1-q)^2$. For the final equality, we used that $q = \frac{\alpha - 1}{\alpha + 1}$. Taking $d \geq R + \frac{1}{2\ln(1/q)}$ ensures that $q^{2d} \leq \frac{q^{2R}}{2}$ so
\begin{align}
\frac{F(\hat{x}) - F^*}{F(0) - F^*} 
&\geq \frac{q^{2R}}{2\alpha} = \frac{\prn*{1 - \frac{2}{\alpha+1}}^{2R}}{2\alpha}
\end{align}
Therefore, 
\begin{equation}
R \leq \frac{\ln\prn*{\frac{F(0) - F^*}{2\alpha\epsilon}}}{\ln\prn*{1 + \frac{2}{\alpha - 1}}} \implies F(\hat{x}) - F^* \geq \epsilon
\end{equation}
Using the fact that $\ln(1+x) \leq x$ and solving the above inequality on $R$ for $\epsilon$, we conclude that
\begin{align}
F(\hat{x}) - F^* 
&\geq \frac{F(0) - F^*}{2\alpha}\exp\prn*{-\frac{2R}{\alpha - 1}}\label{eq:alg-independent-almost-done}
\end{align}
In order to satisfy the strongly convex assumptions, we recall from \eqref{eq:C-for-sdiff} and \eqref{eq:C-for-suboptimality} that we must choose $C$ such that
\begin{align}
\frac{\alpha C^2(H+3\lambda)^2}{64} \leq \frac{\alpha C^2H^2}{16} &\leq \sdiff^2 \\
\frac{qC^2(H-\lambda)}{16} \leq \frac{C^2 H}{16} &\leq \Delta 
\end{align}
Therefore, we choose $C^2 = 16\min\crl*{\frac{\sdiff^2}{\alpha H^2},\, \frac{\Delta}{H}}$ meaning that
\begin{align}
F(\hat{x}) - F^* 
&\geq \frac{F(0) - F^*}{2\alpha}\exp\prn*{-\frac{2R}{\alpha - 1}} \\
&\geq \frac{\min\crl*{\frac{\sdiff^2}{\alpha H},\, \Delta}}{2\alpha}\exp\prn*{-\frac{2R}{\alpha - 1}} \\
&\geq \min\crl*{\frac{\lambda \sdiff^2}{H^2},\, \frac{\sqrt{\lambda}\Delta}{2\sqrt{H}}}\exp\prn*{-\frac{8\sqrt{\lambda}R}{\sqrt{H}}}
\end{align}

For the convex case, we note that in order to satisfy the convex assumptions, we must choose $C$ such that
\begin{align}
\frac{\alpha C^2(H+3\lambda)^2}{64} \leq \frac{\alpha C^2H^2}{16} &\leq \sdiff^2 \\
\frac{\alpha C^2}{4} &\leq B^2 
\end{align}
We therefore choose $C^2 = 4\min\crl*{\frac{\sdiff^2}{\alpha H^2},\, \frac{B^2}{\alpha}}$. Returning to \eqref{eq:alg-independent-almost-done}, this means
\begin{align}
F(\hat{x}) - F^* 
&\geq \frac{F(0) - F^*}{2\alpha}\exp\prn*{-\frac{2R}{\alpha - 1}} \\
&= \frac{qC^2(H-\lambda)}{32\alpha}\exp\prn*{-\frac{2R}{\alpha - 1}} \\
&\geq \frac{q(H-\lambda)\min\crl*{\frac{\sdiff^2}{\alpha H^2}\, \frac{B^2}{\alpha}}}{8\alpha}\exp\prn*{-\frac{8\sqrt{\lambda}R}{\sqrt{H}}} \\
&\geq q \min\crl*{\frac{\sdiff^2}{16\alpha^2 H},\, \frac{HB^2}{16\alpha^2}}\exp\prn*{-\frac{8\sqrt{\lambda}R}{\sqrt{H}}}
\end{align}
From here, we use that $H \geq 7\lambda$ implies $\alpha \geq 2$ so $q \geq 1/3$, so
\begin{align}
F(\hat{x}) - F^* 
&\geq \min\crl*{\frac{\lambda\sdiff^2}{48H^2},\, \frac{\lambda B^2}{48}}\exp\prn*{-\frac{8\sqrt{\lambda}R}{\sqrt{H}}}
\end{align}
Finally, this holds for any $\lambda \geq 0$, so it holds, in particular, for $\lambda = \frac{H}{64R^2}$ thus
\begin{align}
F(\hat{x}) - F^* 
&\geq c\cdot\min\crl*{\frac{\sdiff^2}{HR^2},\, \frac{HB^2}{R^2}}
\end{align}

Finally, it is well known that any first-order method which accesses at most $MKR$ stochastic gradients with variance $\sigma^2$ for a $\lambda$-strongly convex objective will suffer error at least $\beta \frac{\sigma^2}{\lambda MKR}$ in the worst case for a universal constant $\beta$ \cite{nemirovskyyudin1983}. Similarly, any first-order method which accesses at most $MKR$ stochastic gradients with variance $\sigma^2$ for a convex objective with $\nrm{x^*} \leq B$ will suffer error at least $\beta \frac{\sigma B}{\sqrt{MKR}}$ in the worst case for a universal constant $\beta$ \cite{nemirovskyyudin1983}.
\end{proof}

\section{Proof of Corollary \ref{cor:ambsgd-optimal}}\label{app:ambsgd-optimal}
\ambsgdoptimal*
\begin{proof}
In the convex case, Theorem \ref{thm:mbsgd-upper-bound} ensures Accelerated Minibatch SGD converges at a rate proportional to
\begin{equation}
\frac{HB^2}{R^2} + \frac{\sigma B}{\sqrt{MKR}}
\end{equation}
The lower bound for convex functions in Theorem \ref{thm:dzr-lower-bound} precisely matches this whenever 
\begin{equation}
\frac{HB^2}{R^2} = \min\crl*{\frac{\sdiff^2}{HR^2}, \frac{HB^2}{R^2}} \implies \sdiff \geq HB
\end{equation}
For the strongly convex case, Theorem \ref{thm:mbsgd-upper-bound} ensures convergence at a rate porportional to
\begin{equation}
\Delta\exp\prn*{-\frac{\sqrt{\lambda}R}{c_3\sqrt{H}}} + \frac{\sigma^2}{\lambda MKR}\label{eq:corollary-sc-proof}
\end{equation}
The lower bound is given by
\begin{equation}
\min\crl*{\frac{\lambda \sdiff^2}{H^2}, \frac{\Delta\sqrt{\lambda}}{\sqrt{H}}} \exp\prn*{-\frac{8R\sqrt{\lambda}}{\sqrt{H}}} + \frac{\sigma^2}{\lambda MKR} 
\end{equation}
When $\sdiff^2 \geq H^{3/2}/\sqrt{\lambda}$, this reduces to 
\begin{equation}
\frac{\Delta\sqrt{\lambda}}{\sqrt{H}} \exp\prn*{-\frac{8R\sqrt{\lambda}}{\sqrt{H}}} + \frac{\sigma^2}{\lambda MKR} = \Delta \exp\prn*{-\frac{8R\sqrt{\lambda}}{\sqrt{H}} - \log\frac{\sqrt{\lambda}}{\sqrt{H}}} + \frac{\sigma^2}{\lambda MKR} \label{eq:corollary-sc-proof-1}
\end{equation}
Comparing this with \eqref{eq:corollary-sc-proof}, we see that the $R$ needed to guarantee error $\epsilon$ using Theorem \ref{thm:mbsgd-upper-bound} is larger than the minimum possible $R$, as lower bounded by \eqref{eq:corollary-sc-proof-1}, by at most a log factor.
\end{proof}

\section{Discussion of \citet{karimireddy2019scaffold}}\label{app:scaffold}

We compare our results to \citet{karimireddy2019scaffold}, who presented an analysis of the inner/outer stepsize variant of Section \ref{sec:inner-outer} (as \textsc{FedAvg}, with a different stepsize parametrization---see below) as well as the novel method \textsc{SCAFFOLD} which incorporates variance reduction.
\removed{\citeauthor{karimireddy2019scaffold} were concerned mostly with the setting of Section \ref{sec:subset} where only a subset of the machines are used in each round, and thus inter-machine variance reduction is appropriate.  But in order to better understand how the results relate, let us begin by first considering the simpler setting where all machines are used in each round (i.e.~$S=M$).}

\citet[Theorem V]{karimireddy2019scaffold} show that for the inner/outer stepsize updates \eqref{eq:inner-outer-updates}, with optimal choice of stepsizes, in the weakly convex case
\begin{equation}\label{eq:fedavg-bound}
\E F(\hat{x}) - F^* \leq O\left(
\frac{HB^2}{R} + \frac{\sigma B}{\sqrt{SKR}} + \frac{\prn*{H\sdiff^2B^4}^{1/3}}{R^{2/3}} + \sqrt{1-\frac{S}{M}}\cdot \frac{\sdiff B}{\sqrt{SR}} \right)
\end{equation}
and in the strongly convex case,
\begin{equation}\label{eq:fedavg-bound-sc}
\E F(\hat{x}) - F^* \leq O\left(
\lambda B^2\exp\prn*{\frac{-\lambda R}{H}} + \frac{\sigma^2 }{\lambda SKR} + \frac{H\sdiff^2}{\lambda^2 R^2} 
+ (1-\frac{S}{M}) \cdot \frac{\sdiff^2}{\lambda S R}
\right)\qquad\textrm{for } R \geq \Omega\prn*{\frac{H}{\lambda}}.
\end{equation}
But these are loose upper bounds: as discussed in Section \ref{sec:inner-outer}, the Minibatch SGD guarantees also apply to the inner/outer variant (by using $\eta_\textrm{inner}=0$).  The Minibatch SGD guarantees \eqref{eq:S-upper-bound} and \eqref{eq:S-upper-bound-sc} can therefor be viewed also as guarantees on the inner/outer variant (i.e.~\citeauthor{karimireddy2019scaffold}'s \textsc{FedAvg}) that improve over \eqref{eq:fedavg-bound} and \eqref{eq:fedavg-bound-sc} in several ways: (a) they avoid the the third terms in \eqref{eq:fedavg-bound} and \eqref{eq:fedavg-bound-sc}; (b) they improve the fourth terms by a factor of $\sqrt{S}$ and $S$ respectively;  and (c) they avoid the requirement $R > H/\lambda$.  

\citeauthor{karimireddy2019scaffold}'s presentation actually uses a different step-size parametrization that does not allow for $\eta_\textrm{inner}=0$: they use $\eta_l=\eta_\textrm{inner}$ and $\eta_g=\eta_\textrm{outer}/\eta_{\textrm{inner}}$.  We prefer the presentation using $\eta_\textrm{inner}$ and $\eta_\textrm{outer}$ in order to emphasize the relationship with Minibatch SGD and in order to explicitly allow $\eta_\textrm{inner}=0$.  But in any case, even using their parametrization, $\eta_l=\eta_\textrm{inner}$ could be taken arbitrarily close to zero making the deviation from Minibatch SGD arbitrarily small.  Indeed, the \citeauthor{karimireddy2019scaffold}'s bounds \eqref{eq:fedavg-bound} and \eqref{eq:fedavg-bound-sc} are obtained when $\eta_{\textrm{inner}}$ is already so small that the algorithm is essentially equivalent to Minibatch SGD.

But \citeauthor{karimireddy2019scaffold}'s main contribution was the presentation of \textsc{SCAFFOLD}, which incorporates machine-specific control iterates that reduce the inter-machine variances.  For \textsc{SCAFFOLD},  \citep[Theorem VII]{karimireddy2019scaffold} show that in the weakly convex case\footnote{This is different than the bound stated as \cite[Theorem III]{karimireddy2019scaffold}, but is what was proven \cite[Theorem VII, Appendix E]{karimireddy2019scaffold}.}:
\begin{equation}\label{eq:scaffold-bound-S}
\E F(\hat{x}) - F^* \leq O\left(
\frac{HB^2}{R} +  \frac{\sigma B}{\sqrt{SKR}} + \frac{M\sdiff^2}{HSR} + \frac{\sigma \sdiff\sqrt{M}}{HS\sqrt{KR}} \right),
\end{equation}
and in the strongly convex case
\begin{equation}\label{eq:scaffold-bound-S-sc}
\E F(\hat{x}) - F^* \leq O\left(
\lambda\prn*{B^2 + \frac{M\sdiff^2}{SH^2}}\exp\prn*{-\min\crl*{\frac{\lambda}{H},\frac{S}{M}}R} +  \frac{\sigma^2}{\lambda SKR} \right) \qquad\textrm{for }R\geq \max\crl*{\frac{H}{\lambda}, \frac{M}{S}}.
\end{equation}
These guarantees are also obtained when $\eta_{\textrm{inner}}$ is so close to zero that this is essentially a minibatch method, in this case ``Minibatch SAGA'' \citep[cf.][]{defazio2014saga}.

Although \textsc{SCAFFOLD} is aimed specifically at the setting where a subset of machines are used in each round (i.e.~$S<M$), let us first check whether it provides benefits in our ``standard'' setting (introduced in Sections \ref{sec:setup}), where all machines are used each round (i.e.~$S=M$).  In this case, the \textsc{SCAFFOLD} bounds \eqref{eq:scaffold-bound-S} and \eqref{eq:scaffold-bound-S-sc} may improve over the loose upper bounds \eqref{eq:fedavg-bound} and \eqref{eq:fedavg-bound-sc}, but this is only due to the looseness in these upper bounds.  The \textsc{SCAFFOLD} upper bounds (for $S=M$) do not actually improve over the Minibatch SGD guarantees of Theorem \ref{thm:mbsgd-upper-bound}, and only include additional terms (see Tables \ref{tab:prior-local-analysis-convex} and \ref{tab:prior-local-analysis-strongly-convex}).  That is, the \textsc{SCAFFOLD} upper bounds, when $S=M$, are valid also for \textsc{FedAvg} (since as discussed above, Minibatch SGD gurantees are valid also for \textsc{FedAvg}) and so do not show a benefit in the setting where all machines are used each round.  This is perhaps not surprising, since in this setting there is no need to reduce inter-machine variance, and so no benefit from variance reduction.

Let us turn then to the setting of Section \ref{sec:subset}, where only a random subset of $S<M$ machines are used in each iteration, and for which \textsc{SCAFFOLD} was developed.  In this setting, \textsc{SCAFFOLD} does show a benefit in some regimes.  Let us focus on the weakly convex case and compare the \textsc{SCAFFOLD} guarantee \eqref{eq:scaffold-bound-S} with the Minibatch SGD guarantee \eqref{eq:S-upper-bound}.  We can verify that, e.g.~when $\sigma=0$ and $\sdiff=HB$, SCAFFOLD improves over Minibatch SGD if $\frac{M}{R} \ll \frac{S}{M} \ll \frac{R}{M}$, and $S<M$, but the \textsc{SCAFFOLD} guarantee \eqref{eq:scaffold-bound-S} is worse then Minibatch SGD if $\frac{S}{M} \ll \frac{M}{R}$.  More generally, the \textsc{SCAFFOLD} guarantee is worse than Minibatch SGD if $\sigma=0$ and $\frac{S}{M} \ll \frac{\sdiff^2}{H^2B^2} \min(1,\frac{M}{R})$.  Also for strongly convex objectives, \textsc{SCAFFOLD} improves over Minibatch SGD in some regimes but the guarantee \eqref{eq:scaffold-bound-S-sc} is worse than Minibatch SGD in other regimes.  Care is required to map out the precise regimes and how they depend on the various problem parameters.

\end{document}